\documentclass{article}

\PassOptionsToPackage{numbers, compress}{natbib}

 \usepackage[main, final]{neurips_2026}
\usepackage{algorithm}
\usepackage{algpseudocode}
\usepackage{amsmath, amssymb, bm}
\usepackage{bbding}

\usepackage[utf8]{inputenc} 
\usepackage[T1]{fontenc}    
\usepackage{hyperref}       
\usepackage{amsfonts}       
\usepackage{nicefrac}       
\usepackage{microtype}      
\usepackage{xcolor}         

\usepackage{makecell}
\usepackage{siunitx}
\usepackage{animate}
\usepackage{graphicx}	
\usepackage{amsmath}	
\usepackage{amssymb}	
\usepackage{amsthm}
\usepackage{booktabs}
\usepackage{times}
\usepackage{epsfig}
\usepackage{caption}
\usepackage{float}
\usepackage{placeins}
\usepackage{color, colortbl}
\usepackage{stfloats}
\usepackage{enumitem}
\usepackage{tabularx}
\usepackage{xstring}
\usepackage{multirow}
\usepackage{xspace}
\usepackage{url}
\usepackage{subcaption}
\usepackage[hang,flushmargin]{footmisc}
\usepackage{adjustbox}
\usepackage{wrapfig}
\usepackage{algorithm,algorithmicx,algpseudocode}
\usepackage{listings}
\usepackage{bm}

\usepackage{pifont}
\usepackage{thmtools,thm-restate}

\definecolor{tabfirst}{rgb}{1, 0.7, 0.7}
\definecolor{tabsecond}{rgb}{1, 0.85, 0.7}
\definecolor{tabthird}{rgb}{1, 1, 0.7}

\newcommand{\ab}{\mathbf{a}}

\newcommand{\hb}{\mathbf{h}}

\newcommand{\sbb}{\mathbf{s}}

\newcommand{\xb}{\mathbf{x}}

\newcommand{\zb}{\mathbf{z}}

\newcommand{\Dc}{{\mathcal{D}}}
\newcommand{\Ec}{{\mathcal{E}}}

\newcommand{\Nc}{{\mathcal{N}}}

\newcommand{\Ed}{\mathbb{E}}

\newcommand{\Rd}{\mathbb{R}}

\newcommand{\bepsilon}{{\boldsymbol{\epsilon}}}

\DeclareMathOperator{\Var}{Var}

\DeclareMathOperator*{\argmax}{arg\,max}

\def\[#1\]{\begin{align}#1\end{align}}

\newcommand{\ie}{\textit{i}.\textit{e}., }
\newcommand{\eg}{\textit{e}.\textit{g}., }

\usepackage[most]{tcolorbox}

\newtcolorbox[auto counter]{finding}[1][]{%
  enhanced,
  breakable,
  colback=white,
  colframe=black!25,
  boxrule=0.6pt,
  arc=2pt,
  left=6pt,right=6pt,top=5pt,bottom=5pt,   
  before skip=4pt,
  after skip=4pt,
  borderline west={2pt}{0pt}{black!25},
  before upper={\textbf{Finding~\thetcbcounter.}~},
  #1
}

\newtheorem{theorem}{Theorem}
\newtheorem{proposition}[theorem]{Proposition}

\title{Stitched Value Model for Diffusion Alignment}

\author{%
\vspace{-1.0cm}\\
\textbf{Hyojun Go}$^{1,2}$ \quad \textbf{Hyungjin Chung}$^{3*}$\quad \textbf{Prune Truong}$^{2}$ \quad \textbf{Goutam Bhat}$^{2}$ \vspace{0.1cm}\\
\textbf{Li Mi}$^{1}$ \quad \textbf{Zhaochong An}$^{4}$ \quad \textbf{Zixiang Zhao}$^{1 *}$ \quad \textbf{Dominik Narnhofer}$^{1}$ \vspace{0.1cm} \\ 
\textbf{Serge Belongie}$^{4}$ \quad \textbf{Federico Tombari}$^{2}$ \quad \textbf{Konrad Schindler}$^{1}$ \vspace{0.3cm} \\
$^{1}$ETH Zurich \quad $^{2}$Google \quad $^{3}$Korea University \quad $^{4}$University of Copenhagen
\vspace{0.3cm} \\ 
{\normalfont\fontsize{10}{10}\selectfont\hypersetup{pdfborder={0 0 0}}%
\href{https://gohyojun15.github.io/StitchVM/}%
{\textbf{Project page:}\enspace\nolinkurl{gohyojun15.github.io/StitchVM}}}
}

\usepackage[table]{xcolor}

\definecolor{GoogleRed}{HTML}{DB4437}
\definecolor{GoogleGreen}{HTML}{0F9D58}
\definecolor{GoogleBlue}{HTML}{4285F4}

\definecolor{mygrey}{RGB}{235,235,235}
\usepackage{colortbl}

\begin{document}
\begingroup
  \renewcommand{\thefootnote}{*}
  \begin{NoHyper}
    \footnotetext[0]{Corresponding authors:
      Zixiang Zhao (zixiang.zhao@ethz.ch),
      Hyungjin Chung (hj\_chung@korea.ac.kr)}
  \end{NoHyper}
\endgroup
\maketitle

\begin{abstract}
Reward-based alignment of diffusion and flow models requires estimating the value function at noisy intermediate latents, although pretrained reward models operate on clean images.
Common approaches use Tweedie-style clean-image proxies or Monte Carlo rollouts, introducing approximation bias or substantial computational cost.
Directly learning the value function offers an alternative, but matching the performance and generalization of foundation-scale pixel-space reward models on noisy latents remains challenging.
We address this gap by transferring pretrained reward capabilities to noisy latents.
We propose \textbf{StitchVM} (Stitched Value Model), a model stitching framework that connects a frozen diffusion-backbone prefix to a pretrained reward-model suffix through a compatible feature interface.
A closed-form feature fit selects the interface, and lightweight distillation adapts the stitched model, combining noise-aware diffusion features with pretrained reward capability.
For example, adapting CLIP ViT-L with SD~3.5 Medium at $512\times512$ resolution takes approximately 10 GPU-hours on a single GH200.
Once trained, StitchVM provides an amortized alternative to Tweedie-style and Monte Carlo approximations of the value function: it estimates value directly from noisy latents, avoiding clean-image reconstruction and additional denoising rollouts for each evaluation.
We demonstrate benefits across inference-time steering and post-training.
StitchVM improves the quality and efficiency of FK steering and DRaFT, accelerates DPS by $3.7\times$ with half the peak GPU memory, and accelerates DiffusionNFT by $2.3\times$.
\end{abstract}
\section{Introduction}
Diffusion~\citep{ho2020denoising, sohl2015deep, song2021scorebased, go2023addressing} and flow-based~\citep{lipman2023flow, albergo2023building, liu2023flow} denoising models have enabled remarkable success in generative image modelling, including image~\citep{labs2025flux, saharia2022photorealistic, wu2025qwen}, video~\citep{wan2025wan, wiedemer2025video, an2026video, an2025onestory}, and 3D generation~\citep{go2026texttod, go2025splatflow, go2025videorfsplat}. 
Still, the pretraining objective of these models captures the training data distribution, and in practice, task-specific adaptation is often required, e.g., to ensure fidelity to a user prompt~\citep{ghosh2023geneval} or to match human aesthetic preferences~\citep{liang2025aesthetic, wu2023human}. This customization is achieved through {\em alignment}, which aims to adapt the pretrained diffusion or flow model according to a specific reward.

Most existing alignment methods, whether applied at training time~\citep{prabhudesai2023aligning, clark2024directly, lee2023aligning, dong2023raft, wallace2024diffusion, yang2024using, liu2026improving, black2024training, fan2023reinforcement, zheng2026diffusionnft} or at inference time~\citep{chung2023diffusion, song2023loss, ye2024tfg, yu2023freedom, he2024manifold, kim2025flowdps, song2023pseudoinverse, singhal2025a, kim2026inferencetime, li2024derivative, wu2023practical, kim2025testtime, li2025dynamic, zhang2025inferencetime, skreta2025feynmankac}, share a common requirement: they must repeatedly assess noisy latents $\zb_t$ along the denoising trajectory to determine how promising they are.
This information is captured by a \emph{value function}~\citep{uehara2025inference, li2024derivative}, which measures the \emph{expected} reward of clean samples induced by $\zb_t$.

Directly evaluating the value function is difficult, in large part because the reward is normally defined for clean images $\xb_0$~\citep{wu2023human, xu2023imagereward, wang2025unified,radford2021learning, ma2025hpsv3}. 
Therefore, existing methods must resort to workarounds: \textbf{(1)~Tweedie approximation}, which first estimates the posterior mean of the clean sample induced by $\zb_t$, then computes the reward for that proxy~\citep{chung2023diffusion, song2023loss, efron2011tweedie}; or \textbf{(2)~Monte Carlo (MC) approximation}, which rolls out multiple denoising trajectories from $\zb_t$ and averages the reward for each resulting clean sample~\citep{uehara2025inference, li2024derivative}.
Both approaches have significant drawbacks: (1) the Tweedie approximation can be substantially biased in the high-noise regime~\citep{zhu2024think}, moreover, it requires an extra denoiser evaluation and VAE decoding; (2) MC incurs high, often prohibitive, cost for the rollouts.
 
\begin{figure}[t]
    \centering
    \includegraphics[width=1\linewidth]{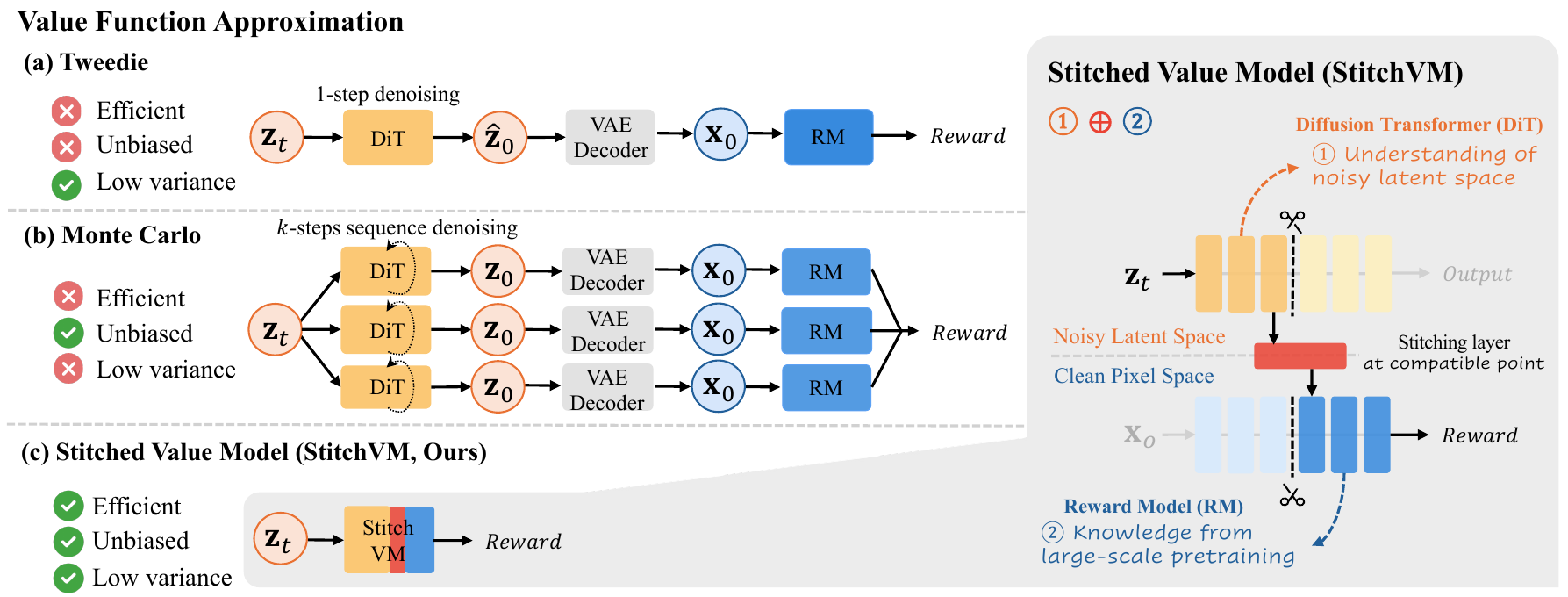}
    \vspace{-0.4cm}
    \caption{\textbf{StitchVM overview.} \textbf{Left:} Unlike Tweedie (a), which requires a denoiser and VAE decoder evaluations and is biased in high noise, and MC (b), which requires $N$-denoising rollouts, StitchVM (c) directly evaluates the value function on noisy latents. \textbf{Right:} StitchVM stitches a diffusion backbone head to a reward model tail, turning the reward model into a value model. 
      }
    \vspace{-0.2cm}
    \label{fig:overview}
\end{figure}

An alternative to these workarounds is to directly learn a value function for noisy latents~\citep{li2024derivative, dai2025vard, liu2026beyond, mi2025video, vysotskyi2026critic}. 
Once trained, such value models can be incorporated into both training-time and inference-time methods, improving alignment along both axes.
In terms of accuracy, they avoid the bias of Tweedie and the variance of stochastic MC rollouts~\citep{vysotskyi2026critic}; in terms of efficiency, they eliminate both the extra denoiser and decoder evaluations required by Tweedie and the costly rollouts of MC~\citep{mi2025video, liu2026improving}.

Despite these clear advantages, only few works have explored direct training of a value model. This is because substantial amounts of data and compute would be required to train a value function for noisy latents that could rival the performance and generality of contemporary pixel-space reward models~\citep{wu2023human, xu2023imagereward, wang2025unified,radford2021learning, ma2025hpsv3}. Beyond the prohibitive upfront cost, such an approach is fundamentally unsustainable: for each new diffusion backbone or improved reward model, one would have to repeat the full large-scale training. Therefore, existing works train at much smaller scales, either reusing diffusion features~\citep{xian2026consistent, zhang2026diffusion, liu2026beyond, mi2025video} or initializing with pretrained reward models~\citep{zhang2024confronting, liang2025aesthetic, zhao2026latsearch, vysotskyi2026critic} to reduce cost. Unfortunately, this leads to inferior accuracy and generalization compared to foundation-scale reward models defined in pixel space~\citep{wu2023human, xu2023imagereward, wang2025unified, radford2021learning, ma2025hpsv3}.
Consequently, the trend has been to fall back to Tweedie or MC approximations, whereas direct value models were largely sidelined.

Here, we propose \textbf{StitchVM} (\textbf{Stitched Value Model}), a framework that transfers the capabilities of pretrained reward models into the noisy latent regime with only a small finetuning cost.
Building on model stitching~\citep{lenc2015understanding, csiszarik2021similarity, yang2022deep, bansal2021revisiting, pan2023stitchable}, our approach combines a truncated frozen diffusion backbone as "head"---natively able to handle noisy latents~\citep{lee2025decoupled, xian2026consistent}---with a sliced pretrained 
reward model as "tail", via a lightweight stitching layer 
(Fig.~\ref{fig:overview}). 
The key is to identify a stitch point where the \emph{representations are compatible}. One way to ensure that is to find layers where the diffusion features of the head can (almost) be mapped to the reward features of the tail with a linear transformation. Since the mapping can be fitted in closed form and the remaining representation gap is small, a short finetuning is sufficient to close the gap without harming the predictive skill of the reward model.  In this way, the stitched model inherits the capability to predict the reward, but is able to operate directly on noisy latents and thus to serve as a value model.

StitchVM is remarkably effective with a range of different diffusion backbones (SD 3.5 Medium~\citep{esser2024scaling, stabilityai2024sd35}, SD 3.5 Large~\citep{esser2024scaling, stabilityai2024sd35}, FLUX~\citep{blackforestlabs2024flux1dev}) and reward models (DFN-CLIP~\citep{fang2024data}, CLIP~\citep{radford2021learning}, Aesthetic Score Predictor~\citep{schuhmann2022improvedaestheticpredictor}, HPSv2~\citep{wu2023human}).
With only a few unlabeled images and lightweight finetuning, the stitched models retain the benchmark performance of the underlying clean reward models while directly ingesting noisy latents. Notably, transferring ViT-L/14@336px CLIP into an SD 3.5 Medium value function takes only $\approx$10 hours on a single GH200 GPU.

We test the stitched value models with various alignment methods.
In case of inference-time alignment, the low-cost estimator for the value function lets each particle in FK steering~\citep{singhal2025a} pick the best of several local proposals at each step, making it more efficient than standard particle scaling. Alternatively, it can replace the long gradient paths of DPS~\citep{chung2023diffusion} with direct gradients from the value model, making the method $3.7\times$ faster and halving peak GPU memory, while at the same time improving quality.
For training-time alignment, our stitched value models enable training at intermediate noisy latents and avoid full rollouts: DiffusionNFT~\citep{zheng2026diffusionnft} becomes $2.3\times$ faster, while direct reward finetuning~\citep{clark2024directly, prabhudesai2023aligning} becomes $1.3\times$ faster, and more effective (e.g., $+30\%$ GenEval) through supervision at high-noise steps.

\section{Related Work}
\label{sec:related_work}
In this section, we briefly review closely related works. See Appendix~\ref{app_sec:extended_relatedwork} for a broad discussion.

\textbf{Alignment methods and value function. }
Most existing inference-time alignment methods evaluate the value function on noisy latents \emph{indirectly}, through approximations, in order to leverage pixel-level rewards.
The \emph{Tweedie approximation}~\citep{chung2023diffusion, song2023loss, efron2011tweedie} forms the basis of many guidance and sequential Monte Carlo methods~\citep{ye2024tfg, yu2023freedom, he2024manifold, kim2025flowdps, song2023pseudoinverse, singhal2025a, kim2026inferencetime, wu2023practical, kim2025testtime, bansal2023universal, han2024trainingfree}, where the estimated clean sample is used either to compute guidance gradients or to weight particles.
The \emph{Monte Carlo approximation}~\citep{uehara2025inference, li2024derivative} instead evaluates the value function by averaging rewards over multiple denoising rollouts, as in SVDD~\citep{li2024derivative} and search-based methods such as DSearch~\citep{li2025dynamic}.
Training-time alignment follows a similar paradigm. Direct reward finetuning~\citep{prabhudesai2023aligning, clark2024directly, wu2024deep} propagates terminal rewards through denoising trajectories, while PPO-style methods~\citep{black2024training, fan2023reinforcement, miao2024training, liu2025flow, xue2025dancegrpo} optimize policy objectives over sampled trajectories.

To avoid these approximations, several works propose to learn the value function directly for the noisy latent~\citep{weng2026realign}. 
These models have been used to improve credit assignment in PPO-style training~\citep{zhang2024confronting, vysotskyi2026critic}, provide feedback at high-noise timesteps in direct reward finetuning~\citep{mi2025video}, and reduce rollout cost in search-based inference~\citep{zhao2026latsearch}.
However, these value models are typically trained with a narrow preference corpus or with task-specific labels, giving reliable signals only in a narrow domain.

\noindent \textbf{Learning value functions and reward models for noisy latents.}
A central challenge in learning value functions and reward models for noisy latents is to retain broad zero-shot generalization without costly large-scale training.
Existing approaches fall into two categories.
\emph{(1) Diffusion-feature predictors} attach prediction heads~\citep{zhang2026diffusion, xian2026consistent, liu2026beyond, mi2025video, guan2026latent, Zhang_2026_CVPR} or LLM interfaces~\citep{bucciarelli2026tiny} to diffusion features.
Although these methods exploit the noise-aware representations of diffusion models~\citep{lee2025decoupled, xian2026consistent}, they must still learn to predict rewards and do not inherit the broad generalization capabilities of pretrained reward models.
\emph{(2) Adaptation of pretrained pixel-space reward models} aims to transfer existing reward prediction capabilities to noisy latents.
One route applies these models to one-step clean-image estimates obtained through Tweedie-style prediction~\citep{liang2025aesthetic}, inheriting the bias of those estimates.
Another learns projections from noisy latents into the reward model's input space~\citep{ramos2025beyond, zhao2026latsearch, zhang2024confronting, vysotskyi2026critic}.
However, the mismatch between these projected inputs and the clean images used to train the reward model introduces a distribution shift.
Recovering the model's predictive capabilities can require substantial adaptation, at most training scales comparable to that of foundational reward models.

The challenge, then, is to transfer the zero-shot capabilities of pretrained pixel-space reward models to noisy latents with only lightweight adaptation.
We address this through \emph{model stitching}, connecting a diffusion model to a pretrained reward model at a compatible representation.
This combines the diffusion model's ability to process noisy latents with the reward model's existing predictive capabilities.
Learning a connection between their intermediate representations enables direct reward prediction from noisy latents while preserving zero-shot performance at a small adaptation.

\noindent\textbf{Model stitching.}
Originally introduced to study neural representations~\citep{lenc2015understanding}, model stitching recomposes the early layers of one neural network and the later layers of another one into a new network, usually with the help of an additional stitching layer. 
Beyond revealing similarities between representations that metrics such as CKA may miss~\citep{csiszarik2021similarity, bansal2021revisiting}, it has been shown that even networks with different architectures can often be stitched into hybrid models with minimal degradation~\citep{kornblith2019similarity}, enabling applications such as resource-constrained model reassembly~\citep{yang2022deep} and 
variable-scale network construction~\citep{pan2023stitchable}. 
Recent work connects pretrained 3D/4D reconstruction models to \emph{clean} video latents through learned feature interfaces~\cite{go2026texttod,an2026vggrpo,liu2026beyond_pixel,hong2026geofix,zhang2026latent,zhang2026aeon,zhang2026splatflux,wang2026vitronedit}.
We extend this idea to the \emph{noisy latent} regime and show that pretrained reward models can also be stitched directly to intermediate states of the denoising process.

\section{Preliminary}
\label{sec:prelim}

\subsection{Diffusion and Flow-based Models}
\label{subsec:diffusion_and_flow}

Let $\xb_0 \in \mathbb{R}^n \sim p_{\rm data}$ denote a clean data sample. We consider the flow matching (FM) framework~\citep{lipman2023flow,albergo2023building,liu2023flow} in latent space~\citep{rombach2022high}, where $\zb_0 = \Ec(\xb_0) \in \Rd^d$ is the clean latent and $\Ec$ is the encoder of the latent diffusion model. Throughout this work, $t = 0$ corresponds to the clean latent distribution $p_0$ (with $\zb_0 \sim p_0$) and $t = 1$ corresponds to the reference Gaussian $p_1 = \Nc(0, I_d)$ (with $\zb_1 = \bepsilon \sim p_1$); note that this is the reverse of the convention in~\citep{lipman2023flow}. We define a Gaussian conditional probability path:
\begin{align}
\label{eq:forward_path}
    \zb_t = \alpha_t\zb_0 + \sigma_t\bepsilon, \quad \bepsilon \sim \Nc(0, I_d) \Leftrightarrow p_t(\zb_t|\zb_0) = \Nc(\zb_t; \alpha_t\zb_0, \sigma_t^2 I_d),
\end{align}
where for FM, $\alpha_t = 1 - t, \sigma_t = t$. This induces a marginal probability path $p_t(\zb_t)$ which interpolates between $p_0$ and $p_1$. FM models learn the marginal velocity field $u_t(\zb_t) = \int u_t(\zb_t|\zb_0)p_{0|t}(\zb_0|\zb_t)\,d\zb_0$,
where $u_t(\zb_t|\zb_0) = \bepsilon - \zb_0$ is the conditional velocity. To sample, one can resort to ODE $\frac{d}{dt}\zb_t = u_t(\zb_t)$, SDE~\citep{song2021scorebased}, or discrete transition kernels~\citep{ho2020denoising,holderrieth2025glass}. See Appendix~\ref{appendix:subsec:reparam} for further discussion.

\subsection{Alignment as reward tilting}
\label{subsec:alignment_as_reward_tilting}
Pretraining aims to model the data distribution. In many applications, however, we do not simply want likely samples—we seek samples that also score highly under some reward function\footnote{In practice, the reward function may include further inputs such as a prompt, we omit them for brevity.} $r(\xb_0): \Rd^n \mapsto \Rd$ that encodes task-specific notions of sample quality, including prompt alignment~\citep{radford2021learning}, aesthetics~\citep{schuhmann2022improvedaestheticpredictor}, human preference~\citep{wu2023human,xu2023imagereward}, and physical consistency~\citep{uehara2025inference,chung2023diffusion, park2025steerx}.
A standard way to formalize alignment is through the reward-tilted target distribution~\citep{uehara2025inference}
\vspace{-0.1cm}
\begin{align}
\label{eq:reward_tilted_dist}
    p^\star(\xb_0)
    =
    \frac{1}{Z_\xb}p(\xb_0)\exp\!\left(r(\xb_0)\right)  \quad \mbox{or, equivalently} \quad
    p^\star(\zb_0) = \frac{1}{Z_\zb} p(\zb_0)\exp\left(
    r\left(\Dc(\zb_0)\right)
    \right),
\end{align}
where $p(\xb_0), p(\zb_0)$ are the base prior distributions from pretraining, $\Dc$ is the decoder, and $Z_\xb, Z_\zb$ are the partition functions.
While the reward is normally defined after decoding with $\Dc$, we often omit it and simply denote $r(\zb_0)$ for simplicity.
Although the generation corresponds to a trajectory through time, starting at $t=1$, the reward is only defined at the terminal $t = 0$. Therefore, it is useful to define the \emph{soft value function}:
\begin{align}
\label{eq:value_fn}
    V_t(\zb_t) :=
    \log\Ed\!\left[\exp(r(\zb_0)) \mid \zb_t\right],
\end{align}
where the expectation is over $\zb_0 \sim p_{0|t}(\zb_0 \mid \zb_t)$.
Value functions can be used in both inference-time steering and post-training, as discussed next.

\noindent\textbf{Inference with gradient guidance.}
One can show (see Appendix~\ref{appendix:subsec:gradient_guidance}) that by modifying the velocity:
\begin{align}
\label{eq:gradient_guidance}
    u_t^r(\zb_t) = u_t(\zb_t) + c_t \nabla_{\zb_t} V_t(\zb_t),
\end{align}
with $c_t$ some constant, one can sample from the tilted distribution in Eq.~\eqref{eq:reward_tilted_dist}. As $V_t(\zb_t)$ is intractable, widely used gradient guidance methods~\citep{chung2023diffusion,bansal2023universal,yu2023freedom} leverage Tweedie approximation, \ie $V_t(\zb_t) \approx r(\Ed[\zb_0|\zb_t])$, which incurs a bias known as the Jensen gap~\citep{chung2023diffusion}.

\noindent\textbf{Inference with particle sampling. }
Methods based on sequential Monte Carlo and search-based methods~\citep{kim2026inferencetime, li2024derivative, wu2023practical, kim2025testtime, li2025dynamic, zhang2025inferencetime, skreta2025feynmankac} evaluate the approximated value function of each particle and probabilistically decide whether to keep the particle or not. Several works again resort to the Tweedie approximation~\citep{wu2023practical,kim2025testtime,singhal2025a}. Others~\citep{li2024derivative} approximate the value function with Monte Carlo (MC) samples, \ie $V_t(\zb_t) \approx \log\big(\frac{1}{N}\sum_{i=1}^N \exp\big(r(\zb_{0,i})\big)\big), \zb_{0,i} \sim p_{0|t}(\zb_0|\zb_t)$. The former approaches again introduce bias, whereas the MC sampling leads to \emph{high variance}, and requires a lot of compute.

\noindent\textbf{Training with reinforcement learning (RL). }
The KL-regularized RL objective
\begin{align}
\label{eq:rl_objective}
    \argmax_{\theta} \Ed_{\zb_0 \sim p_\theta} [r(\zb_0)] - D_{KL}(p_\theta||p)
\end{align}
yields the tilted distribution in Eq.~\eqref{eq:reward_tilted_dist}.
Diffusion sampling can be regarded as a Markov Decision Process~\citep{black2024training}. Existing works that leverage RL post-training for diffusion models aim to optimize for Eq.~\eqref{eq:rl_objective} or variants of it~\citep{clark2024directly,liu2025flow,zheng2026diffusionnft}. Similar to inference-time methods, RL post-training also requires evaluation of the value function, which is normally approximated through MC roll-outs that tend to be unstable and incur high variance. See Appendix \ref{appendix:subsec:rl} for a discussion.
\section{Methodology}

In this section, we present \textbf{StitchVM}, a stitching-based framework that inherits the strong capability of pretrained reward models in the noisy latent regime at small finetuning cost (Section~\ref{sec:method_ StitchVM}). 
We then show how StitchVM improves inference-time (Section~\ref{sec:method_inference}) and training-time (Section~\ref{sec:method_training}) alignment.

\subsection{StitchVM}
\label{sec:method_ StitchVM}

Diffusion backbones natively process noisy latents and extract useful features from them~\citep{lee2025decoupled, xian2026consistent}; while pretrained reward models, trained at foundation model scale, output precise, task-relevant rewards for a broad range of clean images. 
StitchVM combines the two through a lightweight stitching layer that aligns the diffusion features with the reward model's feature space. Specifically, given stitching indices $(i,j)$, the stitched value model is defined as:
\vspace{-0.1cm}
\begin{equation}
\label{eq:stitching_model_architecture}
    V_\omega^{(i, j)}(\zb_t)
    =
    r_\phi^{\geq j}
    \left(
        s_\psi\left(
            u_\theta^{\leq i}(\zb_t)
        \right)
    \right),
\end{equation}
where $u_\theta^{\leq i}$ and $r_\phi^{\geq j}$ denote the diffusion backbone truncated at layer $i$ and the reward model starting from layer $j$, respectively, and $s_\psi$ is the stitching layer.

\noindent\textbf{Stage 1: Selecting the stitching interface. }
A key decision is at which indices $(i,j)$ to stitch, \ie where to hand over from the diffusion model to the reward model.
To identify an interface with compatible representations, we exhaustively search a set of candidate indices.
Given a clean image $\xb_0$ and its latent $\zb_0=\Ec(\xb_0)$, we sample $\zb_t \sim p_t(\zb_t|\zb_0)$ using Eq.~\eqref{eq:forward_path} and extract paired features $u_\theta^{\leq i}(\zb_t)$ and $r_\phi^{\leq j-1}(\zb_0)$.
For each candidate pair $(i,j)$, we fit a linear mapping $W$ by feature matching:
\vspace{-0.2cm}
\begin{equation}
\label{eq:stitch_fit}
    W_{i,j}^{\star}
    =
    \arg\min_W
    \mathbb{E}_{\zb_0,t,\bepsilon}
    \left[
        \left\|
            W u_\theta^{\leq i}(\zb_t)
            -
           r_\phi^{\leq j-1}(\zb_0)
        \right\|_2^2
    \right].
\end{equation}
The optimization can be solved in closed form, making it practical to evaluate many candidate pairs.
We then select the pair $(i^\star,j^\star)$ with the lowest feature-matching loss.

\noindent \textbf{Stage 2: Finetuning StitchVM. }
Since $(i^\star, j^\star)$ are already chosen such that the representations are maximally compatible, the diffusion features (after linear transformation) lie close to the features expected by $r_\phi^{\geq j}$, leaving only a small mismatch. 
A short finetuning of the stitching layer $s_\psi$ and the truncated reward model $r_\phi^{\geq j}$ suffices to compensate for that mismatch, without degrading the reward model's performance.

We finetune the stitched model using unlabeled clean images $\zb_0$.
For each $\zb_0$, we sample a noisy latent $\zb_t \sim p_t(\zb_t|\zb_0)$ from the forward process and use the score $r_\phi(\zb_0)$ of the original reward model as the supervision target:
\vspace{-0.2cm}
\begin{equation}
    \mathcal{L}_{\mathrm{value}}(\omega)
    =
    \mathbb{E}_{\zb_0,\, t,\, \bepsilon}
    \left[
        \left\|
           V_\omega^{(i^\star, j^\star)}(\zb_t)
            -
            r_\phi(\zb_0)
        \right\|_2^2
    \right].
    \label{eq:value_training}
\end{equation}

One can show that the minimizer of Eq.~\eqref{eq:value_training} satisfies $V_{\omega^\star}^{(i^\star,j^\star)}(\zb_t) = \Ed[r_\phi(\zb_0)\mid\zb_t]$, \ie the value function. 
It is worth mentioning some design choices regarding Eq.~\eqref{eq:value_training}. 
First, while on-policy regression is possible~\citep{li2024derivative}, we opt for an off-policy objective to further save compute\footnote{Off-policy and on-policy objectives have the same minimizer when the diffusion policy is exact.}. 
Second, we choose to regress the standard value function rather than the soft value in Eq.~\eqref{eq:value_fn}, as this resulted in more stable training.
One can show that, in terms of the reward scale, the two match to leading order. See Appendix~\ref{appendix:off_policy_value_model_training} for the proofs and further discussion.

Further details for the stitching architecture are given in Appendix~\ref{appendix:svm_details}.

\subsection{Inference-Time Alignment with StitchVM}
\label{sec:method_inference}
StitchVM eliminates the additional denoiser and decoder evaluations required by the Tweedie approximation (Fig.~\ref{fig:overview}). 
We show that, due to this saving, it improves Diffusion Posterior Sampling (DPS)~\citep{chung2023diffusion} in both quality and speed, and Feynman-Kac (FK) steering~\citep{singhal2025a} in quality.

\noindent\textbf{DPS. }
DPS modifies the denoising velocity using the gradient of the value function $\nabla_{\zb_t} V_t(\zb_t)$, as in Eq.~\eqref{eq:gradient_guidance}.
Since the true value function is intractable, previous DPS schemes rely on the Tweedie approximation to estimate this gradient.
As a result, the guidance signal inherits the bias of Tweedie and requires a long backpropagation chain through the reward model, the VAE decoder, and the denoiser.
We instead use the gradient of StitchVM, $\nabla_{\zb_t} V_\omega^{(i^\star, j^\star)}(\zb_t)$, computed directly in the noisy latent space. 
This yields more accurate guidance, especially in high-noise regions while avoiding long backpropagation chains, making DPS both more effective and more efficient. Note the similarity to classifier guidance~\citep{dhariwal2021diffusion}.

\noindent \textbf{FK steering. }
Given particles $\{\zb_{t_k}^n\}_{n=1}^N$, FK steering draws proposals $\bar{\zb}_{t_{k-1}}^n \sim p_{t_{k-1}|t_k}(\zb_{t_{k-1}}^n \mid \zb_{t_k}^n)$ and computes potentials using a value function, e.g.,
$G(\zb_{t_k}^n, \bar{\zb}_{t_{k-1}}^n)=\exp(V_t(\bar{\zb}_{t_{k-1}}^n) - V_t(\zb_{t_k}^n))$.
The next particles are then obtained by resampling according to:
\begin{align}
\label{eq:FK_steering}
    a_{t_{k-1}}^n \sim \mathrm{Multinomial}(G(\zb_{t_k}^1, \bar{\zb}_{t_{k-1}}^1), \ldots, G(\zb_{t_k}^N, \bar{\zb}_{t_{k-1}}^N)), \quad 
    \zb_{t_{k-1}}^n=\bar{\zb}_{t_{k-1}}^{a_{t_{k-1}}^n}.
\end{align}
In text-to-image generation, FK steering estimates the value function via the Tweedie approximation, requiring a denoiser evaluation and VAE decoding for each particle.

In contrast, StitchVM evaluates the value function with a single forward pass of $V_\omega^{(i^\star, j^\star)}$, avoiding both full denoiser inference and VAE decoding.
Value function estimates become substantially cheaper, allowing us to increase the number of proposals per particle without significantly inflating compute.
This exposes a new scaling axis: rather than increasing the number of particles $N$, we can increase the number of local proposals $M$ per particle and select among them using the cheap StitchVM score.
Concretely, at a designated set of steps, each particle spawns $M$ proposals $\zb_{t_{k-1}}^{n,m} \sim p_{t_{k-1}|t_k}(\zb_{t_{k-1}}^n \mid \zb_{t_k}^n)$ for $m=1,\ldots,M$. We then select the best proposal under StitchVM, $\bar{\zb}_{t_{k-1}}^n = \zb_{t_{k-1}}^{n,m_n^\star}$ with $m_n^\star = \arg\max_{m} V_\omega^{(i^\star,j^\star)}(\zb_{t_{k-1}}^{n,m})$.
For further details, see Appendix~\ref{appendix:fksteering}.

\subsection{Training-Time Alignment with StitchVM}
\label{sec:method_training}

Training-time alignment methods require full denoising rollouts to evaluate the reward at the final, clean image. 
Our StitchVM instead enables rollouts to stop at intermediate, noisy latents while still providing supervision through direct evaluation of the value function.
We show how this improves and accelerates direct reward finetuning~\citep{prabhudesai2023aligning,clark2024directly}, and also accelerates DiffusionNFT~\citep{zheng2026diffusionnft}.

\noindent \textbf{AlignProp \& DRaFT. }
To maximize the objective in Eq.~\eqref{eq:rl_objective}, direct reward finetuning methods like AlignProp~\citep{prabhudesai2023aligning} and DRaFT~\citep{clark2024directly} roll out the full denoising trajectory to obtain $\xb_0$ and backpropagate the reward gradient through the chain. 
In practice, this backpropagation is memory-intensive and prone to unstable or exploding gradients, so existing methods often restrict it to the final, low-noise steps. Our StitchVM avoids both issues: at each training iteration, we randomly sample a stopping timestep $\tau$, halt the denoising at $\zb_\tau$, and backpropagate using the StitchVM prediction $V_\omega^{(i^\star, j^\star)}(\zb_\tau)$ in place of the terminal reward.
This delivers effective supervision even in high-noise regions while avoiding long backpropagation chains. For further details, see Appendix~\ref{appendix:value_based_drft}.

\noindent \textbf{DiffusionNFT. }
DiffusionNFT~\citep{zheng2026diffusionnft} performs online RL on the forward process via flow matching, using terminal rewards from complete generations to define positive and negative samples, and thus an implicit direction for improving the policy. 
With StitchVM, we instead stop the generation early at an intermediate, noisy latent $\zb_\tau$, evaluate its value function as $V_\omega^{(i^\star, j^\star)}(\zb_\tau)$, and use that value in place of the terminal reward. 
This allows us to keep the original reward-weighted forward-process regression of DiffusionNFT, while moving supervision from clean, terminal outputs to intermediate, noisy latents. 
This provides supervision without having to generate all the way to a clean sample, substantially improving training efficiency. For further details, see Appendix~\ref{appendix:value_based_nft}.

\section{Experiments}
Here, we show that StitchVM transfers pixel-space reward models into value models on noisy latents while inheriting the reward model's capability (Section~\ref{sec:results_stitched_value_model}).
We then demonstrate that our StitchVM improves both inference-time (Section~\ref{sec:results_inference_time}) and training-time (Section~\ref{sec:results_training_time}) alignment methods.

\begin{figure*}[t!]
\centering
\vspace{-0.1cm}
\begin{subfigure}[t]{\textwidth}
    \centering
    \includegraphics[width=\textwidth]{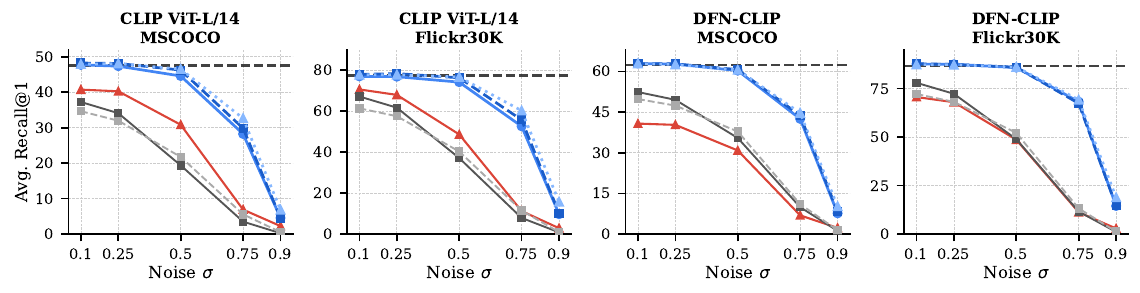}
    \vspace{-0.6cm}
    \caption{Zero-shot image-text retrieval (Avg.\ Recall@1) on MSCOCO and Flickr30K.}
    \vspace{-0.05cm}
    \label{fig:stitchvm_retrieval}
\end{subfigure}
\begin{subfigure}[t]{0.49\textwidth}
    \centering
    \includegraphics[width=\textwidth]{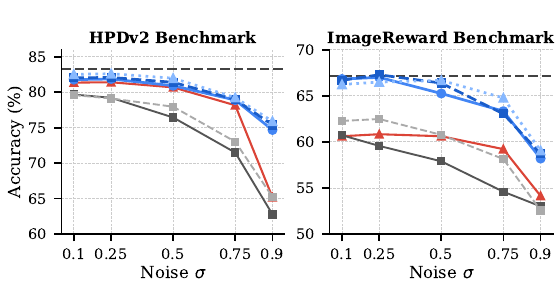}
    \vspace{-0.6cm}
    \caption{Preference accuracy on HPDv2 and ImageReward.}
    \label{fig:stitchvm_preference}
\end{subfigure}%
\hfill
\vspace{-0.2cm}
\begin{subfigure}[t]{0.26\textwidth}
    \centering
    \includegraphics[width=\textwidth]{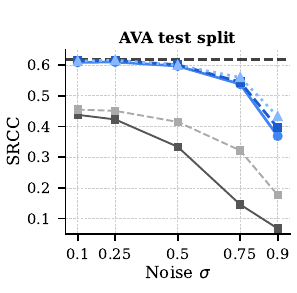}
    \vspace{-0.6cm}
    \caption{Aesthetic SRCC on AVA.}
    \label{fig:stitchvm_aesthetic}
\end{subfigure}%
\hfill
\begin{subfigure}[t]{0.22\textwidth}
    \centering
    \includegraphics[width=\textwidth]{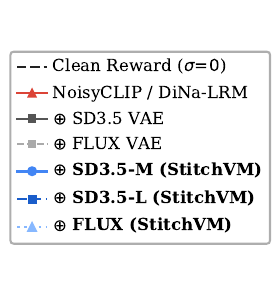}
\end{subfigure}
\vspace{0.2cm}
\caption{\textbf{Results of StitchVM on latents with different noise levels.}
$\oplus$ denotes stitching of a reward model with a pretrained diffusion module
(VAE encoder or DiT).}
\label{fig:stitchvm_results}
\vspace{-0.45cm}
\end{figure*}

\subsection{Main Results: StitchVM Performance}
\label{sec:results_stitched_value_model}

We evaluate the proposed StitchVM across three diffusion backbones---SD 3.5 Medium, SD 3.5 Large~\citep{esser2024scaling, stabilityai2024sd35}, and FLUX.1-dev~\citep{blackforestlabs2024flux1dev}---and four reward models: OpenAI CLIP (ViT-L/14, 336px)~\citep{radford2021learning}, DFN-CLIP (ViT-H/14, 378px)~\citep{fang2024data}, HPSv2~\citep{wu2023human}, and the Aesthetic predictor~\citep{schuhmann2022improvedaestheticpredictor}.
All StitchVMs are trained for 5 epochs on unlabeled images from AVA~\citep{murray2012ava} and HPDv2~\citep{wu2023human}.
We evaluate each model on noisy latents drawn from the flow matching forward process at $\sigma \in \{0.1,0.25,0.5,0.75,0.9\}$.
For CLIP-based models, we report zero-shot cross-modal retrieval on MSCOCO~\citep{lin2014microsoft} and Flickr30K~\citep{young2014image}; for HPSv2, preference accuracy on ImageReward~\citep{xu2023imagereward} and HPDv2~\citep{wu2023human}; and for the Aesthetic predictor, SRCC on the AVA test split following~\citep{hentschel2022clip}.

We compare against three baselines.
First, we adapt VIST3A~\citep{go2026texttod}, which stitches to a VAE decoder rather than the diffusion backbone.
Second, for CLIP-based reward models, we reimplement and train NoisyCLIP~\citep{ramos2025beyond} at scale on LAION-400M~\citep{schuhmann2021laion}.
Third, for HPSv2, we compare with DiNa-LRM~\citep{liu2026beyond}, a diffusion-feature reward model trained on HPDv3 preference data~\citep{ma2025hpsv3}. 
Additional details are provided in Appendix~\ref{appendix:StitchVM_experiments}.

We report the main results in Fig.~\ref{fig:stitchvm_results} (full table in Appendix~\ref{appendix:stitchvm_full_results}) and organize the findings as follows.

\noindent \textbf{(1) StitchVM retains reward model capability on noisy latents. }
At low noise ($\sigma \leq 0.5$), StitchVM closely matches the performance of the original clean reward models across CLIP retrieval, HPSv2 preference prediction, and aesthetic prediction.
As the noise level increases, StitchVM exhibits a gradual performance decline and remains substantially more robust than the baselines.
Thus, StitchVM effectively converts existing pixel-space reward models into noisy latent value models, preserving their original capability while gaining robustness to intermediate noisy latents.

\noindent \textbf{(2) Diffusion features are critical for robust transfer to noisy latents. }
StitchVM substantially outperforms the VAE stitching baseline across all noise levels, with the gap becoming especially large at high noise, where VAE stitching collapses.
This comparison highlights the role of diffusion features: both methods stitch a pretrained reward model into the latent regime, but VAE stitching relies only on clean latent space, whereas StitchVM uses diffusion features that are trained to process noisy latents.
This supports our strategy of stitching diffusion models with reward models.

\noindent \textbf{(3) StitchVM outperforms noisy latent retraining and preference-data training. }
StitchVM outperforms NoisyCLIP~\citep{ramos2025beyond} for CLIP-based reward models, despite NoisyCLIP being trained at a much larger scale on LAION-400M~\citep{schuhmann2021laion}.
This shows that transferring a pretrained reward model through stitching is more effective than retraining the model on noisy latents from scratch.
For HPSv2, StitchVM also outperforms DiNa-LRM~\citep{liu2026beyond} on both HPDv2 and ImageReward, despite using only unlabeled images rather than larger HPDv3 preference dataset~\citep{ma2025hpsv3}.
Together, these comparisons show that StitchVM achieves robust noisy latent value function prediction by transferring the capability of pretrained reward models, without large-scale retraining or preference-label supervision.

\subsection{Main Results: Stitching Vision--Language Reward Models}
\label{sec:vlm_rewards}

Section~\ref{sec:results_stitched_value_model} demonstrates that StitchVM inherits pretrained reward capabilities on noisy latents across image--text retrieval, human preference prediction, and aesthetic prediction.
We now extend this evaluation to reward models that combine a vision encoder with a language model.
We consider HPSv3~\citep{ma2025hpsv3}, based on Qwen2-VL-7B, and UnifiedReward-2.0~\citep{wang2025unified}, with Qwen3.5 backbones of 4B, 9B, and 27B parameters, all stitched to SD~3.5 Medium.

We apply the same stitching principle to the vision encoder: a lightweight stitching layer connects the diffusion features to the retained part of the pretrained vision encoder.
During finetuning, we match the visual features supplied to the language model to those produced from the corresponding clean image by the original reward model.
This confines distillation to the visual branch and avoids the expensive language-model forward and backward passes.
At evaluation, the language model uses these visual features to predict rewards or compare candidate answers, without any changes to its parameters.
All models are trained for five epochs on unlabeled HPDv2 and AVA images.
Additional details are provided in Appendix~\ref{appendix:vlm_setup}.

Following Section~\ref{sec:results_stitched_value_model}, we evaluate how well each stitched model inherits the capability of its original clean-image reward model.
For HPSv3, we report preference accuracy on HPDv2 and ImageReward.
For UnifiedReward, we use VL-RewardBench~\citep{li2025vlrewardbench}, where the model receives an image, a question, and two candidate answers, and generates a judgment of which answer is better.
Table~\ref{tab:vlm_rewards} reports the results across noise levels.

\begin{table}[t]
\centering
\small
\setlength{\tabcolsep}{4pt}
\renewcommand{\arraystretch}{1.1}
\caption{\textbf{Evaluation of vision--language reward models stitched to SD~3.5 Medium.}
Accuracy (in percent, higher is better) across noise levels.}
\label{tab:vlm_rewards}
\begin{tabular}{llcccccc}
\toprule
& & & \multicolumn{5}{c}{\textbf{Stitched model: noise level $\sigma$}} \\
\cmidrule(lr){4-8}
Reward model & Benchmark & Original
& $0.1$ & $0.25$ & $0.5$ & $0.75$ & $0.9$ \\
\midrule
HPSv3 & HPDv2
& \cellcolor{mygrey}85.45
& \cellcolor{GoogleBlue!20}85.00
& \cellcolor{GoogleBlue!20}84.88
& \cellcolor{GoogleBlue!20}84.45
& \cellcolor{GoogleBlue!20}80.91
& \cellcolor{GoogleBlue!20}75.73 \\
HPSv3 & ImageReward
& \cellcolor{mygrey}67.77
& \cellcolor{GoogleBlue!20}68.20
& \cellcolor{GoogleBlue!20}67.79
& \cellcolor{GoogleBlue!20}67.58
& \cellcolor{GoogleBlue!20}64.18
& \cellcolor{GoogleBlue!20}59.73 \\
\midrule
UnifiedReward-4B & VL-RewardBench
& \cellcolor{mygrey}75.30
& \cellcolor{GoogleBlue!20}73.14
& \cellcolor{GoogleBlue!20}72.41
& \cellcolor{GoogleBlue!20}72.09
& \cellcolor{GoogleBlue!20}70.73
& \cellcolor{GoogleBlue!20}69.69 \\
UnifiedReward-9B & VL-RewardBench
& \cellcolor{mygrey}84.36
& \cellcolor{GoogleBlue!20}82.68
& \cellcolor{GoogleBlue!20}83.56
& \cellcolor{GoogleBlue!20}81.96
& \cellcolor{GoogleBlue!20}79.47
& \cellcolor{GoogleBlue!20}79.31 \\
UnifiedReward-27B & VL-RewardBench
& \cellcolor{mygrey}80.91
& \cellcolor{GoogleBlue!20}80.43
& \cellcolor{GoogleBlue!20}80.99
& \cellcolor{GoogleBlue!20}79.79
& \cellcolor{GoogleBlue!20}76.02
& \cellcolor{GoogleBlue!20}73.78 \\
\bottomrule
\end{tabular}
\vspace{-0.2cm}
\end{table}

\noindent \textbf{(1) StitchVM retains the capability of HPSv3 to predict preference. }
At $\sigma=0.25$, the stitched HPSv3 achieves 84.88\% accuracy on HPDv2 and 67.79\% on ImageReward, compared with 85.45\% and 67.77\% for the original model.
The performance remains close to the original model up to $\sigma=0.5$, followed by a larger decline at higher noise levels.
This extends the preference-prediction results of Section~\ref{sec:results_stitched_value_model} to HPSv3: the pretrained language model and reward head can still assess image preference when their visual input is obtained from noisy diffusion latents.

\noindent \textbf{(2) StitchVM also transfers generative reward judgments to noisy latents. }
UnifiedReward evaluates the quality of candidate answers by generating a preference judgment.
At $\sigma=0.25$, the stitched 9B and 27B models achieve 83.56\% and 80.99\% accuracy on VL-RewardBench, compared with 84.36\% and 80.91\% for their original models.
The 4B model exhibits a larger decrease, from 75.30\% to 72.41\%.
These results show that the visual features obtained through stitching retain information that the language model needs to compare answers about an image.
Thus, the capability transfer observed in Section~\ref{sec:results_stitched_value_model} also extends to reward models that generate textual judgments, without finetuning the language model or using preference labels.

\subsection{Results for Inference-Time Methods}
\label{sec:results_inference_time}

We evaluate StitchVM-enhanced DPS and FK steering (Section~\ref{sec:method_inference}) with HPSv2, the Aesthetic predictor, and CLIP-based reward models.
We measure ImageReward, Aesthetic Score, HPSv2, and PickScore~\citep{kirstain2023pick} on generation from DrawBench~\citep{saharia2022photorealistic} prompts, and additionally report GenEval~\citep{ghosh2023geneval} for FK steering.
Additional details are provided in Appendix~\ref{appendix:inference_time_alignment_experiments}.

\begin{table*}[t]
\centering
\scriptsize
\setlength{\tabcolsep}{2pt}
\renewcommand{\arraystretch}{1.1}
\caption{\textbf{Results of DPS with StitchVM on DrawBench.}
ImgRwd: ImageReward, Aes: Aesthetic, Pick: PickScore, Mem: peak GPU memory (GB), Time: seconds per sample.}
\vspace{-0.1cm}
\label{tab:dps_StitchVM}
\resizebox{0.99\textwidth}{!}{%
\begin{tabular}{l *{7}{c} *{7}{c}}
\toprule
& \multicolumn{7}{c}{\textbf{SD3.5-Medium}}
& \multicolumn{7}{c}{\textbf{SD3.5-Large}} \\
\cmidrule(lr){2-8}\cmidrule(lr){9-15}
Method
& ImgRwd & Aes & HPSv2 & Pick & CLIP & Mem $\downarrow$ & Time $\downarrow$
& ImgRwd & Aes & HPSv2 & Pick & CLIP & Mem $\downarrow$ & Time $\downarrow$ \\
\midrule
Flow baseline
 & 0.95 & 5.31 & 0.283 & 22.56 & 28.86 & \textemdash & \textemdash
 & 1.07 & 5.44 & 0.294 & 22.82 & 29.07 & \textemdash & \textemdash \\
\midrule
\rowcolor{mygrey} \multicolumn{15}{l}{\textit{HPSv2 Reward}} \\
\quad DPS
 & 0.96 & 5.27 & 0.335 & 23.02 & 28.94 & 56.4 & 52.8
 & 1.01 & 5.38 & 0.344 & \textbf{23.31} & 28.68 & 89.3 & 84.6 \\
\rowcolor{GoogleBlue!20} \quad DPS + StitchVM
 & \textbf{1.22} & \textbf{5.43} & \textbf{0.348} & \textbf{23.03} & \textbf{28.95} & \textbf{26.0} & \textbf{16.5}
 & \textbf{1.20} & \textbf{5.45} & \textbf{0.356} & 23.13 & \textbf{28.97} & \textbf{43.2} & \textbf{36.3} \\
\midrule
\rowcolor{mygrey} \multicolumn{15}{l}{\textit{Aesthetic Reward}} \\
\quad DPS
 & 0.82 & 5.73 & 0.280 & \textbf{22.47} & 28.16 & 54.3 & 54.2
 & 0.91 & 5.71 & 0.280 & 22.41 & 27.95 & 87.3 & 84.8 \\
\rowcolor{GoogleBlue!20} \quad DPS + StitchVM
 & \textbf{0.98} & \textbf{5.87} & \textbf{0.282} & 22.44 & \textbf{28.43} & \textbf{23.4} & \textbf{14.7}
 & \textbf{1.02} & \textbf{5.76} & \textbf{0.294} & \textbf{22.79} & \textbf{28.80} & \textbf{40.7} & \textbf{36.7} \\
\midrule
\rowcolor{mygrey} \multicolumn{15}{l}{\textit{CLIP Reward}} \\
\quad DPS
 & 0.50 & 4.80 & 0.246 & 21.54 & 33.01 & 54.6 & 52.2
 & 0.72 & 5.04 & 0.263 & 22.03 & 32.55 & 87.5 & 83.1 \\
\rowcolor{GoogleBlue!20} \quad DPS + StitchVM
 & \textbf{0.68} & \textbf{4.82} & \textbf{0.249} & \textbf{21.63} & \textbf{33.12} & \textbf{23.9} & \textbf{14.6}
 & \textbf{1.05} & \textbf{5.20} & \textbf{0.279} & \textbf{22.44} & \textbf{32.95} & \textbf{41.5} & \textbf{36.3} \\
\bottomrule
\end{tabular}}
\vspace{-0.2cm}
\end{table*}

\begin{table*}[t]
\centering
\scriptsize
\setlength{\tabcolsep}{2pt}
\renewcommand{\arraystretch}{1.1}
\caption{\textbf{Results of FK steering with StitchVM on DrawBench and GenEval.}
We set $N=4$ (number of particles).
ImgRwd: ImageReward, Aes: Aesthetic, Pick: PickScore.}
\vspace{-0.1cm}
\label{tab:fk_steering}
\resizebox{\textwidth}{!}{%
\begin{tabular}{l *{5}{c} *{5}{c} *{5}{c}}
\toprule
& \multicolumn{5}{c}{\textbf{SD3.5-Medium}}
& \multicolumn{5}{c}{\textbf{SD3.5-Large}}
& \multicolumn{5}{c}{\textbf{FLUX}} \\
\cmidrule(lr){2-6}\cmidrule(lr){7-11}\cmidrule(lr){12-16}
Method
& ImgRwd & Aes & HPSv2 & Pick & GenEval
& ImgRwd & Aes & HPSv2 & Pick & GenEval
& ImgRwd & Aes & HPSv2 & Pick & GenEval \\
\midrule
Flow baseline
 & 0.88 & 5.34 & 0.282 & 22.55 & 0.62
 & 1.01 & 5.51 & 0.293 & 22.86 & 0.65
 & 1.06 & 5.80 & 0.301 & 22.92 & 0.62 \\
\midrule
\rowcolor{mygrey} \multicolumn{16}{l}{\textit{HPSv2 Reward}} \\
\quad BoN
 & 0.91 & 5.26 & 0.284 & 22.28 & 0.63
 & \textbf{1.24} & 5.41 & 0.308 & 22.96 & 0.68
 & 1.15 & \textbf{5.78} & 0.314 & 23.13 & 0.65 \\
\quad FKS
 & 0.93 & 5.26 & 0.283 & 22.24 & 0.62
 & 1.22 & 5.39 & 0.308 & 22.96 & 0.68
 & 1.17 & 5.76 & 0.314 & 23.10 & 0.66 \\
\rowcolor{GoogleBlue!20} \quad FKS + StitchVM  
 & \textbf{1.10} & \textbf{5.41} & \textbf{0.303} & \textbf{22.86} & \textbf{0.69}
 & 1.20 & \textbf{5.52} & \textbf{0.310} & \textbf{23.11} & \textbf{0.70}
 & \textbf{1.18} & 5.74 & \textbf{0.318} & \textbf{23.22} & \textbf{0.68} \\
\midrule
\rowcolor{mygrey} \multicolumn{16}{l}{\textit{Aesthetic Reward}} \\
\quad BoN
 & 0.73 & 5.47 & 0.270 & 22.12 & 0.56
 & \textbf{1.06} & 5.62 & 0.295 & 22.76 & 0.66
 & \textbf{1.07} & 5.98 & \textbf{0.301} & 22.86 & 0.61 \\
\quad FKS
 & 0.65 & 5.46 & 0.268 & 22.00 & 0.55
 & 1.02 & 5.59 & 0.293 & 22.72 & 0.63
 & 0.99 & 5.94 & 0.299 & 22.82 & 0.60 \\
\rowcolor{GoogleBlue!20} \quad FKS + StitchVM  
 & \textbf{0.99} & \textbf{5.58} & \textbf{0.289} & \textbf{22.65} & \textbf{0.65}
 & 1.03 & \textbf{5.68} & \textbf{0.298} & \textbf{22.87} & \textbf{0.68}
 & 1.01 & \textbf{6.00} & \textbf{0.301} & \textbf{22.88} & \textbf{0.64} \\
\midrule
\rowcolor{mygrey} \multicolumn{16}{l}{\textit{CLIP Reward}} \\
\quad BoN
 & 0.73 & 5.13 & 0.266 & 22.00 & 0.58
 & 1.16 & 5.32 & 0.294 & 22.75 & 0.67
 & 1.09 & 5.73 & 0.301 & 22.94 & 0.63 \\
\quad FKS
 & 0.79 & 5.11 & 0.267 & 22.03 & 0.59
 & 1.18 & 5.30 & 0.295 & 22.79 & 0.68
 & 1.08 & 5.71 & 0.302 & 22.96 & 0.64 \\
\rowcolor{GoogleBlue!20} \quad FKS + StitchVM  
 & \textbf{0.96} & \textbf{5.26} & \textbf{0.282} & \textbf{22.53} & \textbf{0.68}
 & \textbf{1.20} & \textbf{5.40} & \textbf{0.298} & \textbf{22.96} & \textbf{0.71}
 & \textbf{1.11} & \textbf{5.75} & \textbf{0.303} & \textbf{23.00} & \textbf{0.65} \\
\bottomrule
\end{tabular}}
\vspace{-0.3cm}
\end{table*}

\noindent\textbf{(1) StitchVM makes DPS both more effective and more efficient. }
Table~\ref{tab:dps_StitchVM} shows that replacing the Tweedie-approximated DPS gradient with the StitchVM gradient slightly improves DPS across all rewards and all metrics (with two minor exceptions for the PickScore), for both SD~3.5 Medium and SD~3.5 Large.
At the same time, StitchVM substantially reduces inference cost: peak GPU memory drops by about \(50\%\) (\eg \(56.4 \to 26.0\) GB on SD~3.5 Medium), and sampling becomes up to \(3.2\times\) faster (\(52.8 \to 16.5\) s/sample).
These gains come from the same source: StitchVM provides direct noisy latent gradients, avoiding both Tweedie approximation bias in high-noise regions and backpropagation through the denoiser, VAE, and pixel-space reward model.

\noindent \textbf{(2) StitchVM improves FK steering. }
Table~\ref{tab:fk_steering} shows that FK steering with StitchVM outperforms both standard FK steering (FKS) and Best-of-$N$ (BoN) on most metrics, whereas FKS often fails to improve over BoN.
The gains are especially large on SD~3.5 Medium: under HPSv2 reward, FK steering with StitchVM improves ImageReward from $0.93$ (FKS) and $0.91$ (BoN) to $1.10$, and GenEval from $0.62$ to $0.69$.
These gains come from StitchVM's low-cost value function evaluation: each evaluation requires only partial DiT inference up to the stitching layer plus the stitched reward model, and this partial DiT computation is shared with the next denoising step, so the marginal cost of additional proposals is small.
In contrast, achieving the same effect of $M$ under the Tweedie approximation would require a full denoiser inference and VAE decoding for each proposal.

\noindent \textbf{(3) $M$-scaling vs.\ $N$-scaling in FK steering. }
This low marginal cost opens a second scaling axis beyond the standard particle count $N$: each particle spawns $M$ candidates and selects the best under StitchVM.
In Fig.~\ref{fig:m_scaling_fksteering_body}, we evaluate FK steering and its variant with StitchVM on FLUX with HPSv2 target reward by varying $N$ and $M$.
Across the range, FK steering with StitchVM lies above the standard $N$-scaling curve: for example, $(N{=}8, M{=}6)$ matches standard FKS at $N{=}14$ with $33\%$ lower cost.
In other words, using more local proposals $M$ is a more compute-efficient scaling axis than increasing $N$ alone.


\subsection{Results for Training-Time Methods}
\label{sec:results_training_time}

\newsavebox{\rltabbox}
\savebox{\rltabbox}{%
  \begin{minipage}{0.515\textwidth}
    \centering
    \scriptsize
    \setlength{\tabcolsep}{1.6pt}
    \captionof{table}{\textbf{Training-time alignment results, with joint DFN-CLIP + HPSv2 as the training reward.}}
    \vspace{-0.2cm}
    \renewcommand{\arraystretch}{1.10}
    \resizebox{\linewidth}{!}{%
    \begin{tabular}{l c ccccc}
    \toprule
    Method & GPU-h $\downarrow$ & HPSv2 & DFN & ImgRwd & Pick & GenEval \\
    \midrule
    Flow-GRPO-fast
     & 122.2 & 0.348 & 0.408 & 1.44 & 22.77 & 0.65 \\
    \midrule
    DRaFT-1
     & 128.1 & 0.308 & 0.379 & 0.89 & 22.65 & 0.53 \\
    DRaFT-3
     & 128.0 & 0.329 & 0.392 & 1.36 & 21.36 & 0.66 \\
    \rowcolor{GoogleBlue!20} DRaFT-1 + StitchVM
     & \textbf{94.8} & \textbf{0.348} & 0.418 & \textbf{1.47} & 23.06 & 0.69 \\
    \rowcolor{GoogleBlue!20} DRaFT-3 + StitchVM
     & 100.3 & 0.347 & \textbf{0.420} & \textbf{1.47} & \textbf{23.13} & \textbf{0.71} \\
    \midrule
    DiffusionNFT
     & 191.5 & \textbf{0.347} & 0.413 & \textbf{1.50} & 22.98 & 0.67 \\
    \rowcolor{GoogleBlue!20} DiffusionNFT + StitchVM
     & \textbf{84.7} & \textbf{0.347} & \textbf{0.414} & \textbf{1.50} & \textbf{23.06} & \textbf{0.68} \\
    \bottomrule
    \end{tabular}%
    }
    \vspace{-0.3cm}
    \label{tab:rl_finetune}
  \end{minipage}}
\begin{figure}[t]
\centering
\begin{minipage}[c]{0.46\textwidth}
\centering
\includegraphics[width=\linewidth]{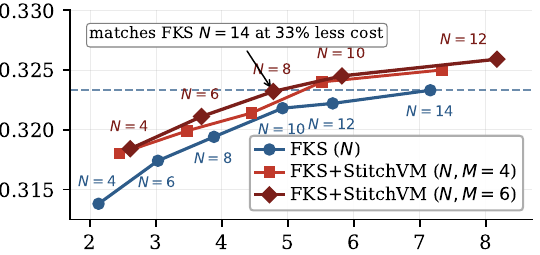}
\vspace{-0.6cm}
\captionof{figure}{\textbf{HPSv2 reward (target) vs.\ GPU-hours over 200 prompts on FK steering.}}
\vspace{-0.3cm}
\label{fig:m_scaling_fksteering_body}
\end{minipage}\hfill
\begin{minipage}[c]{0.52\textwidth}\usebox{\rltabbox}\end{minipage}
\end{figure}

We finetune SD3.5 Medium at \(512{\times}512\) resolution using DFN-CLIP and HPSv2 as training rewards. 
We compare DiffusionNFT and DRaFT-\(K\), which backpropagates through the final \(K\) denoising steps~\citep{clark2024directly}, against their StitchVM-based variants with \(K\in\{1,3\}\). 
We also include Flow-GRPO-Fast~\citep{liu2025flow} as a baseline.
As metrics, we report GenEval and ImageReward, PickScore, HPSv2, and DFN-CLIP scores on DrawBench. 
Additional details are provided in Appendix~\ref{appendix:training_time_alignment_experiments}.

\noindent \textbf{Results. }
Table~\ref{tab:rl_finetune} shows that StitchVM accelerates both DRaFT and DiffusionNFT, while additionally improving DRaFT's generation quality. 
By stopping rollouts at intermediate noisy latents and evaluating the value function directly, StitchVM reduces GPU-hours by 22--26\% for DRaFT and by over 55\% for DiffusionNFT. 
For DRaFT, this also improves quality: standard DRaFT restricts backpropagation to low-noise steps to avoid unstable gradients, whereas StitchVM provides direct value function supervision at intermediate latents, including high-noise regions. 
We further plot training curves in Appendix~\ref{appendix:additional_results_rl_convergence}. 
As shown in the results, StitchVM variants reach higher scores across all metrics with less compute.

\subsection{Analysis}
\label{sec:analysis}

\begin{wraptable}{r}{0.45\linewidth}
\centering
\vspace{-0.5cm}
\scriptsize
\setlength{\tabcolsep}{3pt}
\caption{\textbf{StitchVM training cost (GPU-h).}}
\vspace{-0.2cm}
\renewcommand{\arraystretch}{1.10}
\begin{tabular}{l rrr rrr}
\toprule
& \multicolumn{3}{c}{512$\times$512} & \multicolumn{3}{c}{1024$\times$1024} \\
\cmidrule(lr){2-4}\cmidrule(lr){5-7}
Reward & Train & Search & Total & Train & Search & Total \\
\midrule
Aesthetic & 6.4 & 0.6 & 7.0  & 23.1 & 1.1 & 24.2 \\
CLIP      & 9.3 & 0.7 & 10.0 & 23.5 & 1.0 & 24.5 \\
HPSv2     & 9.4 & 0.8 & 10.2 & 31.0 & 1.3 & 32.3\\
\bottomrule
\end{tabular}
\label{tab:stitchvm_compute}
\vspace{-0.5cm}
\end{wraptable}

\textbf{Training cost of StitchVM. }
Table~\ref{tab:stitchvm_compute} reports the total computational cost of StitchVM on SD~3.5~Medium with different reward models, measured on GH200 GPUs and including both the search for the stitching layer and the subsequent finetuning stage.
Each StitchVM reported in Table~\ref{tab:stitchvm_compute} takes approximately 10 GPU-hours at $512\times512$ resolution, or 24--32 GPU-hours at $1024\times1024$.
The timings underline that the lightweight, one-time transfer procedure of StitchVM is a lot more efficient than the large-scale retraining of a reward model.

\noindent \textbf{Additional results. } In Appendix~\ref{appendix:additional_results}, we further analyze the stitching layer search (Appendix~\ref{appendix:stitching_layer_analysis}), demonstrate that a smaller StitchVM can guide a larger generator (Appendix~\ref{appendix:smallgenerator}), and ablate when to stop rollouts for DiffusionNFT with StitchVM (Appendix~\ref{appendix:stop_step_distribution}).

\section{Conclusion}


We have presented \textbf{StitchVM}, a recipe for transferring high-end pixel-space reward models into value functions for the noisy latents of a diffusion or flow model.
The key idea is to stitch a frozen diffusion backbone to the pretrained reward model. To that end we find the most compatible layers in the two models, which by construction already includes fitting the stitching layer. After that, a light, self-supervised finetuning stage is sufficient to close the remaining representation gap.
The resulting StitchVMs closely match the underlying original reward models but can be evaluated at noisy latent precursors of the generated pixel images.
We have applied the stitched value functions both to inference-time alignment methods (FK steering and DPS) and to training-time alignment (direct reward finetuning and DiffusionNFT).
Across both settings, replacing expensive per-sample approximations with direct evaluation of the StitchVM value function improves efficiency while maintaining or even improving alignment quality.
More broadly speaking, StitchVM provides a generic template for combining a latent diffusion backbone with a pixel-space feedforward model without sacrificing the extensive pretraining of either.
We believe that this practice may have important applications beyond diffusion model alignment.


\section{Acknowledgments}
This work was supported by an unrestricted gift from Google, and by a grant from the Swiss National Supercomputing Centre within the Swiss AI Initiative.
Hyungjin Chung was supported by National Research Foundation of Korea (NRF) grant funded by the Korea government (MSIT) (RS-2025-23523979).
Financial support to Konrad Schindler and Li Mi is acknowledged from Horizon Europe grant No. 101213369 supported by the Swiss State Secretariat for Education, Research and Innovation under contract number 25.00138.
Zhaochong An and Serge Belongie are supported by funding from the Pioneer Centre for AI, DNRF grant number P1.

{\small
   \bibliographystyle{unsrtnat}  
  \bibliography{ref}
}

\clearpage
\appendix


\section{Proofs}
\label{appendix:proof}

We start by reviewing the sampling process of the flow-based models. A standard approach is to resort to ODE sampling:
\begin{align}
    \zb_1 \sim p_1, \quad \frac{d}{dt}\zb_t = u_t(\zb_t), \quad t : 1 \to 0 \Rightarrow \zb_t \sim p_t.
\end{align}
One can also resort to SDE sampling~\citep{song2021scorebased}:
\begin{align}
\label{eq:sde}
    \zb_1 \sim p_1, \quad d\zb_t = \left[
    u_t(\zb_t) - \frac{\nu_t^2}{2}\nabla_{\zb_t} \log p_t(\zb_t)
    \right]dt + \nu_t\, dW_t, \quad t : 1 \to 0,
\end{align}
where $dW_t$ is the standard Wiener process, and $\nabla_{\zb_t} \log p_t(\zb_t)$ is the score function, which can be obtained by simply reparametrizing $u_t$. The diffusion coefficient $\nu_t$ is a free parameter controlling the amount of injected noise; for FM with $\alpha_t = 1-t, \sigma_t = t$, a natural choice that preserves the marginals of the probability flow ODE is $\nu_t^2 = \frac{4t}{1-t}$, giving drift coefficient $\frac{2t}{1-t}$ on the score. See Section~\ref{appendix:subsec:reparam} for the full derivation.
In discrete time, for schedules $1=t_K > t_{K-1} > \cdots > t_0 = 0$, there exist other ways to sample from the data distribution with discrete transition kernels~\citep{ho2020denoising,holderrieth2025glass}, i.e.
\begin{align}
    p_0(\zb_0) = \int p_1(\zb_{t_K}) \left[
    \prod_{k = 1}^{K} p_{t_{k-1}|t_{k}}(\zb_{t_{k-1}}|\zb_{t_k})
    \right]\,d\zb_{t_{1:K}},
\end{align}
where $p_{t_{k-1}|t_{k}}(\zb_{t_{k-1}}|\zb_{t_k})$ is the discrete transition kernel.

\subsection{Reparametrization in Diffusion and Flow-based Models}
\label{appendix:subsec:reparam}
Here, we provide a brief clarification on why the reparametrization of the velocity field in FMs can retrieve the score function $\nabla_{\zb_t} \log p_t(\zb_t)$ and the posterior mean $\Ed[\zb_0|\zb_t]$. For general $\alpha_t$ and $\sigma_t$, the conditional vector field is given as~\citep{lipman2023flow}
\begin{align}
    u_t(\zb_t|\zb_0) = \frac{\dot{\sigma}_t}{\sigma_t}\zb_t + \left(\dot{\alpha}_t - \alpha_t\frac{\dot{\sigma}_t}{\sigma_t}\right)\zb_0.
\end{align}
\paragraph{Denoiser as velocity field reparametrization.}
We have that
\begin{align}
    u_t(\zb_t) &= \int u_t(\zb_t|\zb_0)\,p_{0|t}(\zb_0|\zb_t)\,d\zb_0 \\
    &= \frac{\dot{\sigma}_t}{\sigma_t}\zb_t + \left(\dot{\alpha}_t - \alpha_t\frac{\dot{\sigma}_t}{\sigma_t}\right) \int \zb_0\, p_{0|t}(\zb_0|\zb_t)\,d\zb_0 \\
    &= \frac{\dot{\sigma}_t}{\sigma_t}\zb_t + \left(\dot{\alpha}_t - \alpha_t\frac{\dot{\sigma}_t}{\sigma_t}\right) D_t(\zb_t),
\end{align}
where we denoted the posterior mean as the {\em denoiser} $D_t(\zb_t) := \Ed[\zb_0|\zb_t]$. Rearranging,
\begin{align}
\label{eq:denoiser_from_velocity}
    D_t(\zb_t) &= \frac{1}{\dot{\alpha}_t\sigma_t - \alpha_t\dot{\sigma}_t}\left(\sigma_t u_t(\zb_t) - \dot{\sigma}_t\zb_t\right).
\end{align}
\paragraph{Score function as velocity field reparametrization.}
Tweedie's formula establishes the relation between the posterior mean and the score function
\begin{align}
\label{eq:tweedie}
    D_t(\zb_t) = \Ed[\zb_0|\zb_t] = \frac{1}{\alpha_t}\left(
    \zb_t + \sigma_t^2\nabla_{\zb_t}\log p_t(\zb_t)
    \right).
\end{align}
Plugging Eq.~\eqref{eq:tweedie} into Eq.~\eqref{eq:denoiser_from_velocity} and rearranging yields
\begin{align}
\label{eq:velocity_score_reparam}
    u_t(\zb_t) = \frac{\dot{\alpha}_t}{\alpha_t}\zb_t + \frac{\tilde\nu_t^2}{2}\nabla_{\zb_t} \log p_t(\zb_t),
\end{align}
where we defined
\begin{align}
\label{eq:nu_tilde_def}
    \frac{\tilde\nu_t^2}{2} := \frac{\sigma_t(\dot{\alpha}_t\sigma_t - \alpha_t\dot{\sigma}_t)}{\alpha_t} = \sigma_t^2\frac{d}{dt}\log\frac{\alpha_t}{\sigma_t}.
\end{align}
We use $\tilde\nu_t$ (rather than $\nu_t$) here to distinguish this velocity–score coupling from the SDE diffusion coefficient $\nu_t$ in the main-text sampling SDE~\eqref{eq:sde}. These are different quantities: for the FM schedule below, $\tilde\nu_t^2 < 0$ while $\nu_t^2 > 0$ is a free design parameter of the sampler.

\paragraph{Specialization to the FM schedule.}
We now specialize the above to the FM schedule used in the main text, $\alpha_t = 1-t$ and $\sigma_t = t$, for which $\dot{\alpha}_t = -1$, $\dot{\sigma}_t = 1$, and
\begin{align}
    \dot{\alpha}_t\sigma_t - \alpha_t\dot{\sigma}_t = -t - (1-t) = -1.
\end{align}
The denoiser~\eqref{eq:denoiser_from_velocity} simplifies to
\begin{align}
\label{eq:denoiser_fm}
    D_t(\zb_t) = \zb_t - t\, u_t(\zb_t),
\end{align}
and the velocity–score coupling~\eqref{eq:nu_tilde_def} becomes
\begin{align}
\label{eq:nu_tilde_fm}
    \frac{\tilde\nu_t^2}{2} = -\frac{t}{1-t},
\end{align}
so~\eqref{eq:velocity_score_reparam} reads
\begin{align}
\label{eq:velocity_score_reparam_fm}
    u_t(\zb_t) = -\frac{1}{1-t}\zb_t - \frac{t}{1-t}\nabla_{\zb_t}\log p_t(\zb_t).
\end{align}
Equivalently, the score can be recovered from $u_t$ and $\zb_t$ as
\begin{align}
\label{eq:score_from_velocity_fm}
    \nabla_{\zb_t}\log p_t(\zb_t) = -\frac{(1-t)\,u_t(\zb_t) + \zb_t}{t}.
\end{align}

\subsection{From score reparametrization to gradient guidance}
\label{appendix:subsec:gradient_guidance}

\paragraph{Step 1: Sampling SDE in terms of the score.}
The sampling SDE~\eqref{eq:sde} in the main text,
\begin{align}
\label{eq:sde_restated}
    d\zb_t = \left[u_t(\zb_t) - \frac{\nu_t^2}{2}\nabla_{\zb_t}\log p_t(\zb_t)\right]dt + \nu_t\, dW_t, \quad t : 1\to 0,
\end{align}
is written in terms of the velocity $u_t$. Substituting the velocity–score reparametrization~\eqref{eq:velocity_score_reparam} eliminates $u_t$ and gives a form in which the drift depends only on $\zb_t$ and $\nabla_{\zb_t}\log p_t(\zb_t)$:
\begin{align}
    d\zb_t &= \left[\frac{\dot\alpha_t}{\alpha_t}\zb_t + \frac{\tilde\nu_t^2}{2}\nabla_{\zb_t}\log p_t(\zb_t) - \frac{\nu_t^2}{2}\nabla_{\zb_t}\log p_t(\zb_t)\right]dt + \nu_t\, dW_t \\
\label{eq:sde_score_only}
    &= \left[\frac{\dot\alpha_t}{\alpha_t}\zb_t + \frac{\tilde\nu_t^2 - \nu_t^2}{2}\nabla_{\zb_t}\log p_t(\zb_t)\right]dt + \nu_t\, dW_t.
\end{align}
For the FM schedule with the main-text choice $\nu_t^2 = 4t/(1-t)$, combined with~\eqref{eq:nu_tilde_fm},
\begin{align}
    \frac{\tilde\nu_t^2 - \nu_t^2}{2} = -\frac{t}{1-t} - \frac{2t}{1-t} = -\frac{3t}{1-t},
\end{align}
so~\eqref{eq:sde_score_only} becomes
\begin{align}
\label{eq:sde_score_only_fm}
    d\zb_t = \left[-\frac{1}{1-t}\zb_t - \frac{3t}{1-t}\nabla_{\zb_t}\log p_t(\zb_t)\right]dt + \nu_t\, dW_t.
\end{align}

\paragraph{Step 2: Score of the reward-tilted marginal.}
To sample from the tilted target~\eqref{eq:reward_tilted_dist}, we need the score of the tilted marginal $p^\star_t$ at every intermediate $t$. Marginalizing the same forward kernel $p_t(\zb_t|\zb_0)$ against $p^\star_0$,
\begin{align}
    p^\star_t(\zb_t)
    &= \int p_t(\zb_t|\zb_0)\, p^\star_0(\zb_0)\, d\zb_0 \\
    &= \frac{1}{Z_\zb}\int p_t(\zb_t|\zb_0)\, p_0(\zb_0)\exp(r(\zb_0))\, d\zb_0 \\
    &= \frac{p_t(\zb_t)}{Z_\zb}\int p_{0|t}(\zb_0|\zb_t)\exp(r(\zb_0))\, d\zb_0 \\
\label{eq:tilted_marginal}
    &= \frac{p_t(\zb_t)}{Z_\zb}\exp(V_t(\zb_t)),
\end{align}
where the last line uses the definition of the value function~\eqref{eq:value_fn}. Taking the logarithm and then the gradient in $\zb_t$, the $\zb_t$-independent factor $Z_\zb$ drops out, leaving
\begin{align}
\label{eq:tilted_score}
    \boxed{\;\nabla_{\zb_t}\log p^\star_t(\zb_t) = \nabla_{\zb_t}\log p_t(\zb_t) + \nabla_{\zb_t} V_t(\zb_t)\;}
\end{align}
The score of the tilted marginal is the pretrained score shifted by the gradient of the value function.

\paragraph{Step 3: Gradient guidance as sampling from $p^\star$.}
Sampling from $p^\star$ uses the same SDE~\eqref{eq:sde_score_only} but with the tilted score $\nabla\log p^\star_t$ in place of $\nabla\log p_t$. Substituting~\eqref{eq:tilted_score},
\begin{align}
    d\zb_t 
    &= \left[\frac{\dot\alpha_t}{\alpha_t}\zb_t + \frac{\tilde\nu_t^2 - \nu_t^2}{2}\nabla_{\zb_t}\log p^\star_t(\zb_t)\right]dt + \nu_t\, dW_t \\
    &= \left[\frac{\dot\alpha_t}{\alpha_t}\zb_t + \frac{\tilde\nu_t^2 - \nu_t^2}{2}\bigl(\nabla_{\zb_t}\log p_t(\zb_t) + \nabla_{\zb_t}V_t(\zb_t)\bigr)\right]dt + \nu_t\, dW_t \\
\label{eq:tilted_sde_score}
    &= \underbrace{\left[\frac{\dot\alpha_t}{\alpha_t}\zb_t + \frac{\tilde\nu_t^2 - \nu_t^2}{2}\nabla_{\zb_t}\log p_t(\zb_t)\right]dt + \nu_t\, dW_t}_{\text{pretrained sampling SDE~\eqref{eq:sde_score_only}}} + \underbrace{\frac{\tilde\nu_t^2 - \nu_t^2}{2}\nabla_{\zb_t}V_t(\zb_t)\,dt}_{\text{gradient guidance correction}}.
\end{align}
Translating the score-only form back to the velocity form by inverting the substitution in Step~1 (i.e.\ using $\frac{\dot\alpha_t}{\alpha_t}\zb_t = u_t(\zb_t) - \frac{\tilde\nu_t^2}{2}\nabla_{\zb_t}\log p_t$), we obtain the equivalent expression
\begin{align}
\label{eq:tilted_sde_velocity}
    d\zb_t = \left[u_t(\zb_t) - \frac{\nu_t^2}{2}\nabla_{\zb_t}\log p_t(\zb_t) + \frac{\tilde\nu_t^2 - \nu_t^2}{2}\nabla_{\zb_t}V_t(\zb_t)\right]dt + \nu_t\, dW_t.
\end{align}
That is, sampling from the reward-tilted distribution $p^\star$ reduces to running the pretrained sampling SDE~\eqref{eq:sde_restated} with a single additional drift term proportional to $\nabla_{\zb_t}V_t(\zb_t)$. This is exactly \emph{gradient guidance}: following the gradient of the log value function at each step steers the trajectory toward high-reward regions, and by~\eqref{eq:tilted_score} it does so in precisely the way required to sample from $p^\star$. Specializing to FM with $\nu_t^2 = 4t/(1-t)$, the guidance coefficient becomes $\frac{\tilde\nu_t^2 - \nu_t^2}{2} = -\frac{3t}{1-t}$, so~\eqref{eq:tilted_sde_velocity} reads
\begin{align}
\label{eq:tilted_sde_velocity_fm}
    d\zb_t = \left[u_t(\zb_t) - \frac{2t}{1-t}\nabla_{\zb_t}\log p_t(\zb_t) - \frac{3t}{1-t}\nabla_{\zb_t}V_t(\zb_t)\right]dt + \nu_t\, dW_t.
\end{align}

\subsection{Reinforcement Post-training of Diffusion}
\label{appendix:subsec:rl}

\paragraph{Setup.}
The discrete denoising transition of diffusion sampling can be formulated as a Markov Decision Process (MDP). Let the state at step $k$ be $\zb_{t_k}$, the action is the denoised estimate $\ab_{t_k} := \zb_{t_{k-1}}$ predicted by the neural network, and the policy is defined as $\pi(\ab_{t_k} | \sbb_{t_k}) := p_\theta(\zb_{t_{k-1}}|\zb_{t_k})$. The transition is deterministic, and the initial state distribution is the reference distribution. The reward is defined at the final step, i.e. $t = 0$.

\paragraph{KL-regularized RL induces reward-tilted distribution.}
We show that the KL-regularized RL objective~\eqref{eq:rl_objective} has the reward-tilted distribution~\eqref{eq:reward_tilted_dist} as its unique optimum. Expanding the KL term,
\begin{align}
    \Ed_{\zb_0 \sim p_\theta}[r(\zb_0)] - D_{\rm KL}(p_\theta\,\|\,p)
    &= \int p_\theta(\zb_0)\left[r(\zb_0) + \log p(\zb_0) - \log p_\theta(\zb_0)\right]d\zb_0 \\
    &= \int p_\theta(\zb_0)\log\frac{p(\zb_0)\exp(r(\zb_0))}{p_\theta(\zb_0)}\,d\zb_0 \\
    &= \int p_\theta(\zb_0)\log\frac{Z_\zb\, p^\star(\zb_0)}{p_\theta(\zb_0)}\,d\zb_0 \\
\label{eq:rl_as_kl}
    &= \log Z_\zb - D_{\rm KL}(p_\theta\,\|\,p^\star),
\end{align}
where the third line uses the definition $p^\star(\zb_0) = \frac{1}{Z_\zb}p(\zb_0)\exp(r(\zb_0))$ from~\eqref{eq:reward_tilted_dist}. Since $\log Z_\zb$ is independent of $\theta$ and $D_{\rm KL}(p_\theta\,\|\,p^\star) \geq 0$ with equality iff $p_\theta = p^\star$, the objective is maximized precisely when $p_\theta = p^\star$. Hence optimizing~\eqref{eq:rl_objective} is equivalent to minimizing the reverse KL to the reward-tilted target $p^\star$, and the two formulations share the same global optimum.

\subsection{Off-policy Value Model Training}
\label{appendix:off_policy_value_model_training}

\begin{proposition}
\label{prop:conditional_mean}
Consider the population objective
\begin{align}
    \mathcal{L}_{\mathrm{value}}(\omega)
    =
    \Ed_{t}\,\Ed_{\zb_0 \sim p_0}\,\Ed_{\bepsilon \sim \Nc(0, I_d)}
    \!\left[
        \left(V_\omega^{(i^\star,j^\star)}(\zb_t) - r_\phi(\zb_0)\right)^2
    \right],
\end{align}
with $\zb_t = \alpha_t \zb_0 + \sigma_t \bepsilon$. Assume the parametric family $\{V_\omega^{(i^\star,j^\star)} : \omega\}$ is expressive enough to represent any function of $(t, \zb_t)$. Then the minimizer satisfies
\begin{align}
\label{eq:cond_mean_minimizer}
    V_{\omega^\star}^{(i^\star,j^\star)}(\zb_t) \;=\; \Ed\!\left[r_\phi(\zb_0)\mid\zb_t\right],
\end{align}
where the conditional expectation is taken under the posterior $p_{0|t}(\zb_0|\zb_t) \propto p_t(\zb_t|\zb_0)\, p_0(\zb_0)$.
\end{proposition}

\begin{proof}
Rewrite the objective by first conditioning on $(t, \zb_t)$:
\begin{align}
    \mathcal{L}_{\mathrm{value}}(\omega)
    =
    \Ed_{t}\,\Ed_{\zb_t \sim p_t}\,\Ed_{\zb_0 \sim p_{0|t}(\cdot|\zb_t)}
    \!\left[
        \left(V_\omega^{(i^\star,j^\star)}(\zb_t) - r_\phi(\zb_0)\right)^2
    \right].
\end{align}
For each fixed $(t, \zb_t)$, the model output $V_\omega^{(i^\star,j^\star)}(\zb_t)$ is a single scalar, while $r_\phi(\zb_0)$ varies according to the posterior $p_{0|t}(\cdot|\zb_t)$. The inner expectation is therefore a one-dimensional quadratic in this scalar, which is uniquely minimized by the posterior mean:
\begin{align}
    \arg\min_{c\in\Rd}\,\Ed_{\zb_0 \sim p_{0|t}(\cdot|\zb_t)}\!\left[(c - r_\phi(\zb_0))^2\right]
    \;=\; \Ed[r_\phi(\zb_0)\mid\zb_t].
\end{align}
Since this minimum is attained at every $(t, \zb_t)$ by the same function $\zb_t \mapsto \Ed[r_\phi(\zb_0)\mid\zb_t]$, and the outer expectation is a non-negative weighted average of the inner minima, it is minimized by the same function. The expressiveness assumption ensures this function lies in the parametric family.
\end{proof}

We refer to the {\em standard} value function as
\begin{align}
    \tilde V_t(\zb_t) \;:=\; \Ed[r_\phi(\zb_0)\mid\zb_t],
\end{align}
to distinguish it from the \emph{soft} value function
\begin{align}
    V_t(\zb_t) \;:=\; \log\Ed[\exp(r_\phi(\zb_0))\mid\zb_t]
\end{align}
defined in \eqref{eq:value_fn}.

\paragraph{Relation to the soft value function.}
Applying Jensen's inequality to $\exp$,
\begin{align}
\label{eq:jensen_value}
    \tilde V_t(\zb_t)
    \;=\;
    \Ed[r_\phi(\zb_0)\mid\zb_t]
    \;\leq\;
    \log\Ed[\exp(r_\phi(\zb_0))\mid\zb_t]
    \;=\;
    V_t(\zb_t),
\end{align}
with equality if and only if $r_\phi(\zb_0)$ is constant on the support of the posterior $p_{0|t}(\cdot|\zb_t)$. The two thus coincide in the noiseless limit $t\to 0$, where the posterior concentrates on the corresponding clean latent.

To make the connection precise at finite $t$, consider the tempered reward $r \mapsto \lambda r$ with inverse temperature $\lambda > 0$, and write
\begin{align}
    V^{(\lambda)}(\zb_t) := \log\Ed\!\left[\exp(\lambda r_\phi(\zb_0))\mid\zb_t\right],
    \qquad
    \tilde V^{(\lambda)}(\zb_t) := \Ed[\lambda r_\phi(\zb_0)\mid\zb_t] = \lambda \tilde V_t(\zb_t).
\end{align}
Let $\mu_t := \Ed[r_\phi(\zb_0)\mid\zb_t] = \tilde V_t(\zb_t)$ and $\sigma_t^2 := \Var[r_\phi(\zb_0)\mid\zb_t]$. Expanding the exponential and the logarithm in powers of $\lambda$ around $\lambda = 0$,
\begin{align}
    \Ed[\exp(\lambda r_\phi(\zb_0))\mid\zb_t]
    &=
    1 + \lambda \mu_t + \frac{\lambda^2}{2}(\mu_t^2 + \sigma_t^2) + O(\lambda^3),
\end{align}
and applying $\log(1 + x) = x - x^2/2 + O(x^3)$ yields
\begin{align}
\label{eq:soft_value_taylor}
    V^{(\lambda)}(\zb_t)
    \;=\;
    \lambda \tilde V_t(\zb_t)
    \;+\;
    \frac{\lambda^2}{2}\Var[r_\phi(\zb_0)\mid\zb_t]
    \;+\;
    O(\lambda^3),
\end{align}
and consequently
\begin{align}
\label{eq:soft_value_grad_taylor}
    \nabla_{\zb_t} V^{(\lambda)}(\zb_t)
    \;=\;
    \lambda\,\nabla_{\zb_t}\tilde V_t(\zb_t)
    \;+\;
    \frac{\lambda^2}{2}\,\nabla_{\zb_t}\Var[r_\phi(\zb_0)\mid\zb_t]
    \;+\;
    O(\lambda^3).
\end{align}
Two consequences make $\tilde V$ a faithful surrogate for $V$ in the alignment methods of Section~\ref{subsec:alignment_as_reward_tilting}.

\noindent \emph{Gradient guidance.} Substituting \eqref{eq:soft_value_grad_taylor} into the gradient-guidance correction \eqref{eq:tilted_sde_velocity}, the leading-order tilt is
\begin{align}
    \frac{\tilde\nu_t^2 - \nu_t^2}{2}\,\nabla_{\zb_t} V^{(\lambda)}(\zb_t)
    \;=\;
    \lambda\cdot\frac{\tilde\nu_t^2 - \nu_t^2}{2}\,\nabla_{\zb_t}\tilde V_t(\zb_t)
    \;+\;
    O(\lambda^2),
\end{align}
which is exactly the guidance one obtains by replacing $V$ with $\tilde V$ and absorbing the factor $\lambda$ into the guidance scale $c_t$ in \eqref{eq:gradient_guidance}. To leading order in the reward scale, gradient guidance with the regressed conditional-mean value $\tilde V$ thus samples from the reward-tilted distribution at temperature $\lambda$, with the temperature implicit in the chosen guidance coefficient.

\noindent \emph{Particle methods.} FK steering and other SMC-style methods use $V$ only through pairwise comparisons that determine particle weights (Eq.~\eqref{eq:FK_steering}). By \eqref{eq:soft_value_taylor}, $V^{(\lambda)}$ and $\lambda \tilde V$ differ by $\frac{\lambda^2}{2}\sigma_t^2 + O(\lambda^3)$, so they induce the same particle ordering up to $O(\lambda^2)$ corrections controlled by the conditional reward variance $\sigma_t^2$. The correction is largest in the high-noise regime (where $\sigma_t^2$ is largest), but $\tilde V$ retains the dominant ranking signal in our experiments.


\section{Extended Related Work}
\label{app_sec:extended_relatedwork}

Here, we expand the related work in Section~\ref{sec:related_work} with a detailed discussion of (i) inference-time approximations of the value function, (ii) training-time alignment with Monte Carlo approximation, and (iii) the role of value models and noisy latent reward models in existing alignment pipelines.

\paragraph{Inference-time approximations.}
Inference-time alignment can be viewed as an approximation to a soft optimal denoising policy, in which the value function provides a look-ahead reward prediction~\citep{uehara2025inference}.
Since rewards are typically defined in clean pixel space, the value function cannot be evaluated directly on noisy latents.
The \emph{Tweedie approximation}~\citep{chung2023diffusion, song2023loss, efron2011tweedie} forms a one-step clean estimate via the Tweedie formula, decodes it through the VAE, and evaluates the pixel-space reward on the result.
Most guidance methods~\citep{chung2023diffusion, song2023loss, ye2024tfg, yu2023freedom, he2024manifold, kim2025flowdps, song2023pseudoinverse, bansal2023universal, han2024trainingfree} use the gradient of this approximation to modify the denoising velocity or score, while sequential Monte Carlo (SMC) methods~\citep{singhal2025a, kim2026inferencetime, wu2023practical, kim2025testtime} use it to compute particle importance weights and resampling probabilities.
This strategy incurs two costs: estimator bias in high-noise regions and an additional denoiser pass per evaluation.
The \emph{Monte Carlo approximation}~\citep{uehara2025inference, li2024derivative} instead estimates the value function by rolling out multiple denoising trajectories from a noisy latent and aggregating the resulting rewards, as in SVDD~\citep{li2024derivative} and search-based methods such as DSearch~\citep{li2025dynamic}.
This avoids approximation bias but incurs prohibitive trajectory-rollout cost, which compounds with single-trajectory variance in budget-limited regimes.

\paragraph{Training-time alignment with Monte Carlo approximation.}
RL-based post-training methods~\citep{prabhudesai2023aligning, clark2024directly, lee2023aligning, dong2023raft, wallace2024diffusion, yang2024using, liu2026improving, black2024training, fan2023reinforcement, zheng2026diffusionnft, bao2026renft} similarly require value function estimation along the denoising trajectory, which is commonly approximated by Monte Carlo rollouts.
Direct reward finetuning~\citep{clark2024directly, prabhudesai2023aligning, wu2024deep} backpropagates terminal rewards through full denoising trajectories; PPO-family methods~\citep{black2024training, fan2023reinforcement, miao2024training, liu2025flow, xue2025dancegrpo, ding2026treegrpo, li2026branchgrpo, li2025mixgrpo, wang2025grpo} optimize policy gradients over sampled trajectories; and DiffusionNFT~\citep{zheng2026diffusionnft} uses terminal rewards from complete generations within a forward-process contrastive objective.
These methods inherit the cost of rollout-based training, suffer from high-variance terminal-reward estimates when only one or a few trajectories are used~\citep{vysotskyi2026critic}, and provide only weak credit assignment to intermediate denoising steps~\citep{zhang2024confronting}.

\paragraph{Value models and noisy latent reward models in alignment pipelines.}
Several works leverage value models, or more broadly noisy latent reward models, to address the limitations above~\citep{weng2026realign}.
LatSearch~\citep{zhao2026latsearch} evaluates intermediate noisy latents instead of fully rolling out every candidate trajectory, thereby reducing the cost of search-based inference.
In PPO-style post-training, such models have been used to provide per-step feedback that improves credit assignment over intermediate denoising steps~\citep{zhang2024confronting} and stabilizes high-variance policy-gradient updates~\citep{vysotskyi2026critic}.
DPO-style methods~\citep{liang2025aesthetic, xian2026consistent, zhang2026diffusion} likewise utilize noisy latent rewards or preferences to refine credit assignment over intermediate denoising states, while direct reward finetuning with a value model~\citep{mi2025video, tan2026easytune} provides feedback before full denoising completion, reducing rollout costs.
\section{Additional Methodological Details}
\label{appendix:additional_method}

\subsection{StitchVM Training}
\label{appendix:svm_details}
This subsection provides details that were left implicit in Section~\ref{sec:method_ StitchVM}.
We describe the precise form of the stitching layer $s_\psi$, which extends the closed-form linear component in Eq.~\eqref{eq:stitch_fit}, and the practical restrictions we impose on the stitching-layer search grid.

\paragraph{Stitching layer architecture.}
The closed-form solution $W_{i^\star,j^\star}^\star$ gives the best \emph{linear} map between the two pretrained representation spaces.
To additionally model the nonlinear discrepancy while reusing this least-squares-optimal initialization, we parameterize $s_\psi$ as a residual block built on top of $W_{i^\star,j^\star}^\star$:
\begin{equation}
    s_\psi(\hb)
    \;=\;
    \mathrm{Up}\!\big(F_\psi(\hb)\big)
    \;+\;
    G_\psi\!\Big(\mathrm{Up}\!\big(F_\psi(\hb)\big)\Big),
\end{equation}
where $\hb$ denotes the diffusion-backbone feature at layer $i^\star$, $F_\psi$ is a $1{\times}1$ convolution initialized with $W_{i^\star,j^\star}^\star$, $\mathrm{Up}(\cdot)$ is a deterministic bilinear resampling operator that matches the spatial resolution of the target reward model representation, and
\begin{align}
G_\psi = \mathrm{Conv}_{1{\times}1} \circ \mathrm{SiLU} \circ \mathrm{Conv}_{1{\times}1} \circ \mathrm{SiLU}
\end{align}
is a two-layer pointwise MLP with bottleneck ratio $1{:}r$ ($r=8$ in all stitching configurations).
The weights and bias of the final $1{\times}1$ convolution in $G_\psi$ are zero-initialized, so that $G_\psi(\cdot)\equiv \mathbf{0}$ at the beginning of training.

\noindent\emph{Channel and resolution mismatch.}
The DiT token grid and the reward model patch grid generally have different spatial resolutions, since the two backbones were pretrained with different patch sizes and input resolutions.
We bridge this resolution mismatch with a single bilinear resampling step applied after $F_\psi$, denoted by $\mathrm{Up}(\cdot)$ in the equation above.
The channel mismatch between the diffusion backbone and the reward model is handled entirely by $F_\psi$.

\noindent\emph{Tokenization.}
Since the reward model is truncated at layer $j^\star$, the stitched feature must be converted into the token format expected by the remaining reward model suffix $r_\phi^{\geq j}$.
We perform this token-format adaptation following the reward model's original tokenization scheme.
Specifically, the spatial output of $s_\psi$ is flattened into patch tokens, augmented with the special tokens required by the reward model suffix, and combined with the corresponding positional embeddings before being passed to $r_\phi^{\geq j}$.
For CLIP-based reward models, this amounts to prepending the \texttt{[CLS]} token and adding the positional embeddings.

\noindent\emph{Diffusion-backbone conditioning.}
When the diffusion backbone requires conditioning inputs such as the timestep or text embeddings, we provide the noise-level conditioning corresponding to the forward process and use a fixed null-text conditioning during training.

\paragraph{Stitching layer selection.}
We rank candidate layer pairs $(i,j)$ according to the closed-form feature-matching loss.
We describe below the probe set used to estimate this loss and the practical restrictions imposed on the search grid.

\noindent\emph{Probe set.}
The expectation in Eq.~\eqref{eq:stitch_fit} is approximated by caching paired features on a held-out subset of $N_{\mathrm{probe}}=200$ clean images.
For each clean latent, we draw $t\!\sim\!\mathrm{Unif}[0,1]$ and follow the same forward path as in Eq.~\eqref{eq:forward_path}.
The search configuration for the vision--language reward models is specified separately in Appendix~\ref{appendix:vlm_setup}.

\noindent\emph{Closed-form solution for the stitching layer search.}
We approximate the population objective in Eq.~\eqref{eq:stitch_fit} by its empirical counterpart on a fixed probe set $\{(\xb_0^{(n)}, t^{(n)}, \bepsilon^{(n)})\}_{n=1}^{N_{\rm probe}}$, where each $\zb_t^{(n)} = \alpha_{t^{(n)}} \zb_0^{(n)} + \sigma_{t^{(n)}} \bepsilon^{(n)}$.
Stacking the paired features as columns into matrices
\begin{align}
    F^{(i)}
    &=
    \left[ u_\theta^{\leq i}(\zb_t^{(1)}), \,\ldots,\,
    u_\theta^{\leq i}(\zb_t^{(N_{\rm probe})}) \right]
    \in \mathbb{R}^{d_i^u \times N_{\rm probe}},
    \\
    G^{(j)}
    &=
    \left[ r_\phi^{\leq j-1}(\zb_0^{(1)}), \,\ldots,\,
    r_\phi^{\leq j-1}(\zb_0^{(N_{\rm probe})}) \right]
    \in \mathbb{R}^{d_j^r \times N_{\rm probe}},
\end{align}
the empirical objective reduces to a standard linear least-squares problem
\begin{align}
    \widehat{W}_{i,j}^\star
    =
    \arg\min_{W \in \mathbb{R}^{d_j^r  \times d_i^u}}
    \left\| W F^{(i)} - G^{(j)} \right\|_F^2,
    \label{eq:closed_form_lsq}
\end{align}
whose minimizer admits the closed-form expression
\begin{align}
    \widehat{W}_{i,j}^\star
    =
    G^{(j)} \left( F^{(i)} \right)^{+},
    \label{eq:closed_form_solution}
\end{align}
where $(\cdot)^{+}$ denotes the Moore-Penrose pseudoinverse.

\noindent\emph{Restricting the reward model depth.}
We restrict the search to the early blocks of the reward model and do not sweep over its middle or late blocks.
Empirically, we observed two consistent failure modes as $j$ moves past the early reward model blocks, independent of the diffusion layer $i$:
\begin{enumerate}
    \item the closed-form fitting loss increases sharply, indicating that $W_{i,j}^\star$ can no longer linearly reconstruct the deeper reward model features; and
    \item even after Stage-2 finetuning with $\mathcal{L}_{\mathrm{value}}$, the resulting predictions are substantially worse than those from early-block stitches.
\end{enumerate}
Based on these observations, we select $(i^\star,j^\star)$ within this restricted search space.

\subsection{FK steering with StitchVM}
\label{appendix:fksteering}
Algorithm~\ref{alg:fks_stitchvm} summarizes the FK steering variant with StitchVM.
The algorithm keeps the standard FK potential and resampling step unchanged: at steps $k\in\mathcal{S}_{\mathrm{FK}}$, particles are weighted by the FK potential $G_k$ and resampled accordingly.
The only modification is applied at the proposal-scaling steps $k\in\mathcal{S}_{M}$.
Instead of drawing a single proposal per particle, each particle draws $M$ local proposals from the transition kernel, scores them with StitchVM, and each particle keeps one proposal with the highest StitchVM score.
This nested proposal selection increases the number of candidates explored from $N$ to $NM$ at these steps, while avoiding the additional denoiser and decoder evaluations that Tweedie-based approximation would require.
In our experiments, $\mathcal{S}_{M}$ is chosen as the early denoising steps.

\newcommand{\Vom}{V_{\omega}^{(i^\star,j^\star)}}
\newcommand{\svmblue}[1]{\textcolor{GoogleBlue}{#1}}
\newcommand{\svmcomment}[1]{\textcolor{GoogleBlue}{\(\triangleright\) #1}}
\begin{algorithm}[t]
\caption{FK Steering with StitchVM Proposal Scaling \(M\)}
\label{alg:fks_stitchvm}
\begin{algorithmic}[1]
\Require Transition kernels \(p_{t_{k-1}\mid t_k}\); FK potential function \(G_k\); StitchVM \(\Vom\); number of particles \(N\); proposals per particle \(M\); FK-resampling steps \(\mathcal{S}_{\mathrm{FK}}\subseteq\{1,\dots,K\}\); proposal-scaling steps \(\mathcal{S}_{M}\subseteq\{1,\dots,K\}\); timestep schedule \(t_K>\cdots>t_0\).
\State Sample initial particles \(\zb_{t_K}^n\sim\mathcal{N}(0,I)\) for \(n=1,\dots,N\).
\For{\(k=K,K-1,\dots,1\)}
    \For{\(n=1,\dots,N\)}
        \If{\(k\in\mathcal{S}_{M}\) \textbf{and} \(M>1\)}
            \State \svmblue{\textbf{\# StitchVM proposal scaling:}}
            \State \svmblue{\(\zb_{t_{k-1}}^{n,m}\sim p_{t_{k-1}\mid t_k}(\cdot\mid \zb_{t_k}^n)\) for \(m=1,\dots,M\). \Comment{Sample \(M\) proposals}}
            \State \svmblue{\(m_n^\star=\arg\max_{m\in[M]}\Vom(\zb_{t_{k-1}}^{n,m})\). \Comment{argmax over \(M\) proposals}}
            \State \svmblue{\(\bar{\zb}_{t_{k-1}}^n=\zb_{t_{k-1}}^{n,m_n^\star}\).}
        \Else
            \State Sample one proposal \(\bar{\zb}_{t_{k-1}}^n\sim p_{t_{k-1}\mid t_k}(\cdot\mid \zb_{t_k}^n)\).
        \EndIf
    \EndFor
    \If{\(k\in\mathcal{S}_{\mathrm{FK}}\)}
        \For{\(n=1,\dots,N\)}
            \State Compute \(G^n = G_k(\zb_{t_k}^n,\bar{\zb}_{t_{k-1}}^n)\). \Comment{standard FK potential}
        \EndFor
        \State Sample \(a_{t_{k-1}}^n\sim \mathrm{Multinomial}(G^1,\dots,G^N)\) for \(n=1,\dots,N\).
        \State Set \(\zb_{t_{k-1}}^n \gets \bar{\zb}_{t_{k-1}}^{a_{t_{k-1}}^n}\) for \(n=1,\dots,N\). \Comment{standard FK resampling}
    \Else
        \State Set \(\zb_{t_{k-1}}^n \gets \bar{\zb}_{t_{k-1}}^n\) for \(n=1,\dots,N\).
    \EndIf
\EndFor
\State \Return final particles \(\{\zb_{t_0}^n\}_{n=1}^N\).
\end{algorithmic}
\end{algorithm}

\subsection{AlignProp \& DRaFT with StitchVM}
\label{appendix:value_based_drft}
Direct reward finetuning methods such as AlignProp~\citep{prabhudesai2023aligning} and DRaFT~\citep{dong2023raft} optimize a differentiable reward by generating a clean sample and backpropagating the reward gradient through the denoising trajectory.
We describe this procedure using the latent-space notation of Section~\ref{subsec:diffusion_and_flow}.
Let $\theta$ denote the trainable parameters of the diffusion generator, initialized from the pretrained parameters $\theta_0$.
Starting from $\zb_1\sim p_1=\Nc(0,I_d)$, a full reverse rollout produces a clean sample $\zb_0 \sim p_\theta(\zb_0\mid \zb_1)$.
A standard direct reward finetuning objective can be written as
\begin{align}
    \mathcal{L}_{\mathrm{DRFT}}(\theta)
    =
    -
    \mathbb{E}_{\zb_1\sim p_1,\,
    \zb_0\sim p_\theta(\cdot\mid \zb_1)}
    \left[
        r_\phi(\zb_0)
    \right]
    +
    \lambda\,
    \mathcal{R}(\theta;\theta_0),
    \label{eq:appendix_drft_objective}
\end{align}
where we use the shorthand $r_\phi(\zb_0)=r_\phi(\Dc(\zb_0))$, and $\mathcal{R}(\theta;\theta_0)$ is a regularizer that limits deviation from the pretrained generator, such as an output regularizer~\citep{taghibakhshi2024enhance}.

Because the reward in Eq.~\eqref{eq:appendix_drft_objective} is evaluated only at the terminal clean sample, optimizing it requires differentiating through the full reverse trajectory from $t=1$ to $t=0$.
This is memory-intensive and computationally expensive.
Moreover, backpropagation through a long denoising chain can lead to unstable or exploding gradients in practice~\citep{zheng2026diffusionnft}.
For this reason, direct reward finetuning is often restricted to the final, low-noise portion of the trajectory, which provides weak or indirect learning signals for early, high-noise timesteps.

Our stitched value model $V_\omega^{(i^\star, j^\star)}$ provides a learned estimate of the terminal reward from an intermediate noisy latent.
Instead of sampling all the way to $\zb_0$ and evaluating $r_\phi(\zb_0)$, we stop the reverse process at an intermediate timestep $\tau\in(0,1)$.
That is, we sample $\zb_\tau \sim p_\theta(\zb_\tau\mid\zb_1)$ and replace the remaining rollout from $\zb_\tau$ to $\zb_0$ with a single evaluation of $V_\omega^{(i^\star, j^\star)}$.
This yields the value function-based objective
\begin{align}
    \mathcal{L}_{\mathrm{V\text{-}DRFT}}(\theta)
    =
    -
    \mathbb{E}_{\tau,\,
    \zb_1\sim p_1,\,
    \zb_\tau\sim p_\theta(\zb_\tau \mid \zb_1)}
    \left[
        V_\omega^{(i^\star, j^\star)}(\zb_\tau)
    \right]
    +
    \lambda\,
    \mathcal{R}(\theta;\theta_0),
    \label{eq:appendix_vdrft_objective}
\end{align}
where $\tau$ is sampled from a prespecified stopping-time distribution over the denoising schedule.
This has two practical advantages.
First, it provides direct reward supervision at intermediate noisy states, including high-noise regions where terminal-reward backpropagation is weak or unstable.
Second, it avoids rollouts of the entire denoising trajectory, reducing computation.

\subsection{DiffusionNFT with StitchVM}
\label{appendix:value_based_nft}

DiffusionNFT~\citep{zheng2026diffusionnft} optimizes a reward-weighted forward-process regression objective.
We describe it using the latent-space notation of Section~\ref{subsec:diffusion_and_flow}.
Let $\theta$ denote the current model and let $\theta_{\mathrm{old}}$ denote the frozen data-collection policy.
Starting from $\zb_1\sim p_1$, the data-collection policy generates clean latents $\zb_0$ via the reverse process.
As in Section~\ref{subsec:alignment_as_reward_tilting}, we use the shorthand $r_\phi(\zb_0)=r_\phi(\Dc(\zb_0))$. 

Given samples from $\theta_{\mathrm{old}}$, DiffusionNFT first transforms the raw reward into an optimality probability.
For a normalization group of samples drawn from the data-collection policy, this is written as
\begin{align}
    r_{\mathrm{norm}}(\zb_0)
    =
    \frac{
        r_\phi(\zb_0)
        -
        \mathbb{E}_{\zb'_0 \sim p_{\theta_{\mathrm{old}}}}
        \left[
            r_\phi(\zb'_0)
        \right]
    }{
        Z
    },
    \qquad
    r(\zb_0)
    =
    \frac{1}{2}
    +
    \frac{1}{2}
    \mathrm{clip}
    \left(
        r_{\mathrm{norm}}(\zb_0),
        -1,
        1
    \right),
    \label{eq:appendix_nft_reward_norm}
\end{align}
where the expectation is estimated over the same normalization group and $Z>0$ is a normalization scale\footnote{Following DiffusionNFT's reward-standardization convention~\citep{zheng2026diffusionnft}; the same $Z$ is reused in our StitchVM-based variant (Eq.~\eqref{eq:appendix_vnft_reward_norm}).}.

Given a clean latent $\zb_0$, DiffusionNFT samples a forward noisy state using Eq.~\eqref{eq:forward_path},
\begin{align}
    \zb_t
    =
    \alpha_t\zb_0
    +
    \sigma_t\bepsilon,
    \qquad
    \bepsilon\sim\Nc(0,I_d).
\end{align}
The corresponding conditional velocity, expressed in terms of the generating variables $(\zb_0, \bepsilon)$, is
\begin{align}
\label{eq:appendix_nft_velocity}
    u_t(\zb_t\mid\zb_0)
    =
    \dot{\alpha}_t\zb_0
    +
    \dot{\sigma}_t\bepsilon,
\end{align}
where $\dot\alpha_t$ and $\dot\sigma_t$ denote time derivatives.
This form is equivalent to the $(\zb_t, \zb_0)$-parameterization in Section~\ref{subsec:diffusion_and_flow} and Appendix~\ref{appendix:subsec:reparam}\footnote{Substituting $\bepsilon = (\zb_t - \alpha_t\zb_0)/\sigma_t$ into Eq.~\eqref{eq:appendix_nft_velocity} recovers $u_t(\zb_t\mid\zb_0) = \frac{\dot\sigma_t}{\sigma_t}\zb_t + (\dot\alpha_t - \alpha_t\frac{\dot\sigma_t}{\sigma_t})\zb_0$.}, but is more convenient when $\zb_0$ and $\bepsilon$ are the natural sampling variables, as is the case here.

DiffusionNFT then constructs implicit positive and negative velocity fields by interpolating between the frozen data-collection velocity $u_{\theta_{\mathrm{old}}}$ and the trainable velocity $u_\theta$:
\begin{align}
    u_{\theta}^{+}(\zb_t,t)
    &=
    (1-\beta)\,
    u_{\theta_{\mathrm{old}}}(\zb_t,t)
    +
    \beta\,
    u_{\theta}(\zb_t,t),
    \\
    u_{\theta}^{-}(\zb_t,t)
    &=
    (1+\beta)\,
    u_{\theta_{\mathrm{old}}}(\zb_t,t)
    -
    \beta\,
    u_{\theta}(\zb_t,t).
\end{align}
The DiffusionNFT objective is then constructed as
\begin{align}
    \mathcal{L}_{\mathrm{NFT}}(\theta)
    =
    \mathbb{E}_{\zb_0\sim p_{\theta_{\mathrm{old}}},\,t,\,\bepsilon}
    \Big[
    &\,r(\zb_0)
    \left\|
        u_{\theta}^{+}(\zb_t,t)
        -
        u_t(\zb_t\mid\zb_0)
    \right\|_2^2
    \nonumber\\
    &+
    \bigl(1-r(\zb_0)\bigr)
    \left\|
        u_{\theta}^{-}(\zb_t,t)
        -
        u_t(\zb_t\mid\zb_0)
    \right\|_2^2
    \Big].
    \label{eq:appendix_nft_objective}
\end{align}
After the update, the data-collection policy is updated by an exponential moving average,
\begin{align}
    \theta_{\mathrm{old}}
    \leftarrow
    \rho\,\theta_{\mathrm{old}}
    +
    (1-\rho)\,\theta.
\end{align}

Our StitchVM-based variant retains the same weighted regression structure but replaces the terminal clean reward with a reward model estimate at a stopped noisy latent.
Instead of running the reverse process all the way to $\zb_0$, we stop at an intermediate timestep $\tau\in(0,1)$ and obtain $\zb_\tau$ from the data-collection policy.
The raw scalar signal is then
\begin{align}
    \widetilde r_{\mathrm{raw}}(\zb_\tau)
    =
    V_\omega^{(i^\star, j^\star)}(\zb_\tau).
\end{align}
We convert this estimate into an optimality probability using the same standardization as DiffusionNFT:
\begin{align}
    \widetilde r_{\mathrm{norm}}(\zb_\tau)
    =
    \frac{
        \widetilde r_{\mathrm{raw}}(\zb_\tau)
        -
        \mathbb{E}_{\zb'_\tau \sim p_{\theta_{\mathrm{old}}}}
        \left[
            \widetilde r_{\mathrm{raw}}(\zb'_\tau)
        \right]
    }{
        Z
    },
    \qquad
    \widetilde r(\zb_\tau)
    =
    \frac{1}{2}
    +
    \frac{1}{2}
    \mathrm{clip}
    \left(
        \widetilde r_{\mathrm{norm}}(\zb_\tau),
        -1,
        1
    \right).
    \label{eq:appendix_vnft_reward_norm}
\end{align}

It remains to construct the forward-regression target when the anchor is $\zb_\tau$ rather than $\zb_0$.
For $t>\tau$, the Gaussian path in Eq.~\eqref{eq:forward_path} gives the conditional bridge coefficients
\begin{align}
    \bar{\alpha}_{t|\tau}
    =
    \frac{\alpha_t}{\alpha_\tau},
    \qquad
    \bar{\sigma}_{t|\tau}
    =
    \sqrt{
        \sigma_t^2
        -
        \bar{\alpha}_{t|\tau}^{\,2}
        \sigma_\tau^2
    }.
    \label{eq:appendix_bridge_coefficients}
\end{align}
Starting from the stopped latent $\zb_\tau$, we sample a noisier latent
\begin{align}
    \zb_t
    =
    \bar{\alpha}_{t|\tau}\zb_\tau
    +
    \bar{\sigma}_{t|\tau}\bepsilon,
    \qquad
    \bepsilon\sim\Nc(0,I_d),
    \qquad
    t>\tau,
    \label{eq:appendix_bridge_forward}
\end{align}
with corresponding conditional velocity target
\begin{align}
    u_{t|\tau}(\zb_t\mid\zb_\tau)
    =
    \dot{\bar{\alpha}}_{t|\tau}\zb_\tau
    +
    \dot{\bar{\sigma}}_{t|\tau}\bepsilon,
    \label{eq:appendix_bridge_velocity}
\end{align}
where $\dot{\bar\alpha}_{t|\tau}$ and $\dot{\bar\sigma}_{t|\tau}$ denote time derivatives of the bridge coefficients.
This mirrors the forward-regression construction in Eq.~\eqref{eq:appendix_nft_velocity}, with the anchor replaced by $\zb_\tau$.

The value function-based DiffusionNFT objective is therefore
\begin{align}
    \mathcal{L}_{\mathrm{V\text{-}NFT}}(\theta)
    =
    \mathbb{E}_{\zb_\tau\sim p_{\theta_{\mathrm{old}}},\,t>\tau,\,\bepsilon}
    \Big[
    &\,\widetilde r(\zb_\tau)
    \left\|
        u_{\theta}^{+}(\zb_t,t)
        -
        u_{t|\tau}(\zb_t\mid\zb_\tau)
    \right\|_2^2
    \nonumber\\
    &+
    \bigl(1-\widetilde r(\zb_\tau)\bigr)
    \left\|
        u_{\theta}^{-}(\zb_t,t)
        -
        u_{t|\tau}(\zb_t\mid\zb_\tau)
    \right\|_2^2
    \Big].
    \label{eq:appendix_vnft_objective}
\end{align}

In summary, the original DiffusionNFT uses the terminal clean reward $r_\phi(\zb_0)$ to weight positive and negative forward-regression targets anchored at $\zb_0$.
Our StitchVM-based variant replaces this terminal reward with $V_\omega^{(i^\star, j^\star)}(\zb_\tau)$ and anchors the forward regression at the stopped latent $\zb_\tau$.
Since $V_\omega^{(i^\star, j^\star)}(\zb_\tau)$ estimates the expected terminal reward conditioned on the intermediate state, this construction yields a training signal at intermediate denoising steps without completing the reverse process to $\zb_0$.

\section{Additional Experimental Details}
\label{appendix:additional_experiment}

\subsection{Stitched Value Model Experiments}
\label{appendix:StitchVM_experiments}

\paragraph{StitchVM Implementation details.}
For searching the stitching-layer index, we extract paired features $\big(u_\theta^{\leq i}(\zb_t,t),\, r_\phi^{\leq j-1}(\zb_0)\big)$ for each candidate pair $(i,j)$.
The probe set contains $N_{\mathrm{probe}}=200$ HPDv2 images with $t\!\sim\!\mathrm{Unif}[0,1]$.
We cast the cached features to \texttt{float64} and solve $W^\star_{i,j}$ in closed form using \texttt{torch.linalg.lstsq}.
We restrict the reward model side of the search to the first four transformer blocks ($j\le4$) for all backbone-reward combinations.
On the diffusion side, we sweep all DiT-block indices $i$.

We finetune the stitching layer $s_\psi=(F_\psi,G_\psi)$ together with the truncated reward model suffix $r_\phi^{\geq j}$.
The diffusion backbone remains frozen.
Optimization uses fused AdamW~\citep{loshchilov2018decoupled} with base learning rate $1\!\times\!10^{-5}$, weight decay $0$, and $100$ linear warmup steps.
The stitching layer $s_\psi$ uses an additional $5\times$ learning-rate multiplier on top of the base learning rate, since $F_\psi$ is initialized from the SVD fit and $G_\psi$ is zero-initialized, whereas the reward model suffix already starts from a strong pretrained state and only requires mild adaptation.
We use a global batch size of $128$ images per optimizer step.
Noise levels during StitchVM training are drawn from a center-biased distribution over $\sigma\!\in\![0,1]$, since both endpoints carry little learning signal.

We optimize Eq.~\eqref{eq:value_training} as a multi-level distillation between the stitched model and the original frozen reward model evaluated on the same clean image.
The scalar reward-regression objective in Eq.~\ref{eq:value_training} provides a conceptual formulation of value prediction.
In practice, we transfer reward-model capabilities through multi-level distillation from the original frozen reward model evaluated on the same clean image.
The objectives below serve as practical surrogates for scalar reward regression.

The exact loss depends on the reward model.
For CLIP, DFN-CLIP, and HPSv2, the final score is computed from the inner product between an image embedding and a text embedding.
However, our distillation is performed entirely on the image side: the text encoder is not run or updated during StitchVM training.
We therefore distill the reward model at the representation level:
\begin{align}
    \mathcal{L}_{\mathrm{value}}
    =
    \ell\!\left(
        \tilde e_{\mathrm{StitchVM}},\,
        \tilde e_{\mathrm{RM}}
    \right)
    +
    \lambda_{\mathrm{tok}}\,
    \ell\!\left(
        h_{\mathrm{StitchVM}},\,
        h_{\mathrm{RM}}
    \right) ,
    \label{eq:appendix_clip_distill_loss}
\end{align}
where $\ell(\cdot,\cdot)$ is a regression loss specified below, $h_{\mathrm{StitchVM}}$ and $h_{\mathrm{RM}}$ are the per-token output activations of $r_\phi^{\geq j}$ for the stitched model and the original frozen reward model, respectively, $e_{\mathrm{StitchVM}}$ and $e_{\mathrm{RM}}$ are the corresponding image embeddings, and $\tilde e_{\mathrm{StitchVM}}=e_{\mathrm{StitchVM}}/\|e_{\mathrm{StitchVM}}\|_2$ and $\tilde e_{\mathrm{RM}}=e_{\mathrm{RM}}/\|e_{\mathrm{RM}}\|_2$ are their $\ell_2$-normalized versions.
We instantiate $\ell$ as the squared $\ell_2$ loss and set $\lambda_{\mathrm{tok}}=1$.

For the aesthetic predictor, the final score is a scalar produced by a fixed MLP head $s_{\mathrm{aes}}(\cdot)$ applied to the normalized CLIP image embedding.
We therefore add an explicit score-level term on top of the token and embedding losses:
\begin{align}
    \mathcal{L}_{\mathrm{value}}
    =
    \lambda_{\mathrm{tok}}\,
    \ell\!\left(
        h_{\mathrm{StitchVM}},\,
        h_{\mathrm{RM}}
    \right)
    +
    \lambda_{\mathrm{emb}}\,
    \ell\!\left(
        \tilde e_{\mathrm{StitchVM}},\,
        \tilde e_{\mathrm{RM}}
    \right)
    +
    \lambda_{\mathrm{sc}}\,
    \ell\!\left(
        s_{\mathrm{aes}}(\tilde e_{\mathrm{StitchVM}}),\,
        s_{\mathrm{aes}}(\tilde e_{\mathrm{RM}})
    \right) .
    \label{eq:appendix_aesthetic_distill_loss}
\end{align}
The notation for $h_{\mathrm{StitchVM}}, h_{\mathrm{RM}}, \tilde e_{\mathrm{StitchVM}}, \tilde e_{\mathrm{RM}}$ matches Eq.~\eqref{eq:appendix_clip_distill_loss}.
For the aesthetic case, we instantiate $\ell$ as the Smooth-$\ell_1$ loss and set $\lambda_{\mathrm{tok}}=\lambda_{\mathrm{emb}}=\lambda_{\mathrm{sc}}=1$.
In all cases, the original reward model parameters are frozen, and for the aesthetic predictor, the score head $s_{\mathrm{aes}}$ is also frozen.

The noisy latent retrieval and preference benchmarks in Table~\ref{tab:main_stitchvm_results} use $1024{\times}1024$ source images. 
For the training alignment experiments in Table~\ref{tab:rl_finetune}, we used StitchVM trained on $512{\times}512$ resolution to match the image resolution in~\citep{zheng2026diffusionnft, liu2025flow}.

\paragraph{Baseline implementation details.}
The VAE-stitching baseline (rows ``$\oplus$ SD3 VAE'' and ``$\oplus$ FLUX VAE'' in Table~\ref{tab:main_stitchvm_results} and Fig.~\ref{fig:stitchvm_results}) follows the same training recipe as our StitchVM, but stitches at the \emph{VAE-encoder level} rather than through the diffusion backbone, in the spirit of VIST3A~\citep{go2026texttod} and VGGRPO~\citep{an2026vggrpo}.
Concretely, we replace the diffusion head $u_\theta^{\leq i}$ in Eq.~\eqref{eq:stitching_model_architecture} with the identity on the noisy VAE latent.
The noisy latent $\zb_t$ is fed directly into a $1{\times}1$ stitching convolution and a residual MLP, then bilinearly resampled to the patch resolution expected by the reward model suffix $r_\phi^{\geq j}$.
The original VIST3A and VGGRPO formulations stitch at $t\!=\!0$ using clean latents.
We extend the same architecture to the noisy latent regime by sampling $\zb_t$ from the same forward process used for StitchVM.
Thus, the only conceptual difference between this baseline and our StitchVM is whether the diffusion DiT, and hence noise-aware diffusion features, is present in the front end.
The stitching layer is initialized from the closed-form least-squares fit between $\zb_t$ and $r_\phi^{\leq j-1}(\zb_0)$.
This is analogous to Eq.~\eqref{eq:stitch_fit}, but with $u_\theta^{\leq i}=\zb_t$.

We finetune the stitching layer and the reward model suffix for $15$ epochs.
This is longer than the $5$ epochs used for StitchVM because the absence of diffusion features makes the optimization landscape harder.
All other optimization hyperparameters match StitchVM: AdamW with base learning rate $1\!\times\!10^{-5}$, a $5\times$ multiplier on the stitching layer, weight decay $0$, $100$ linear warmup steps.
Resolutions follow the standard image-input size of each base model.
The distillation losses are identical to those in Eqs.~\eqref{eq:appendix_clip_distill_loss}--\eqref{eq:appendix_aesthetic_distill_loss}.

For DiNa-LRM~\citep{liu2026beyond}, we use the original released checkpoint without any additional training or finetuning.

\subsection{Vision--Language Reward Model Experiments}
\label{appendix:vlm_setup}

\paragraph{StitchVM implementation details.}
For the experiments in Section~\ref{sec:vlm_rewards}, we stitch SD~3.5 Medium to the vision encoder of HPSv3~\citep{ma2025hpsv3} and UnifiedReward-2.0~\citep{wang2025unified}.
HPSv3 uses Qwen2-VL-7B, while the three UnifiedReward models use Qwen3.5 backbones with 4B, 9B, and 27B parameters.
We use the stitching layer $s_\psi=(F_\psi,G_\psi)$ described in Appendix~\ref{appendix:svm_details}, retaining the remaining vision blocks and the module that projects their outputs into the language model's input space.
This module combines each $2\times2$ group of neighboring visual tokens into one output token.
We preserve the original vision encoder's spatial token order and positional information.
The language model and reward output head remain unchanged.

For stitching-layer selection, we use 300 images from the HPDv2 and AVA training sets: 200 images for the closed-form feature fit and 100 for validation.
We consider the outputs of the first 12 DiT blocks and the first four vision-encoder blocks, giving 48 candidate pairs.
For each pair, we fit the affine mapping using Eq.~\eqref{eq:stitch_fit} and select the pair with the lowest validation feature-matching loss.
Noise levels are sampled uniformly from $[0,1]$.
The selected models retain the first five DiT blocks for HPSv3, four for UnifiedReward-4B, and three for UnifiedReward-9B and -27B.
In all four cases, the diffusion features replace the outputs of the first two vision-encoder blocks.

\paragraph{Distillation objective.}
As in Appendix~\ref{appendix:StitchVM_experiments}, we transfer the pretrained reward capability through image-side feature matching.
Here, $h_{\mathrm{StitchVM}}$ and $h_{\mathrm{RM}}$ denote the vision-encoder outputs supplied to the language model, including the final projection described above.
They are computed from a noisy latent and its corresponding clean image, respectively.
We optimize
\begin{align}
    \mathcal{L}_{\mathrm{value}}
    &=
    \ell\!\left(
        h_{\mathrm{StitchVM}},\,
        h_{\mathrm{RM}}
    \right),
    \label{eq:appendix_vlm_distill_loss}
    \\
    \ell(A,B)
    &=
    \frac{1}{Nd}\|A-B\|_F^2,
    \qquad A,B\in\mathbb{R}^{N\times d},
    \nonumber
\end{align}
where $N$ is the number of visual output tokens and $d$ their feature dimension.
The loss is averaged over the training batch.
We finetune the stitching layer, the retained vision blocks, and their final projection, while keeping the diffusion backbone, VAE, and original vision encoder frozen.
The language model is not executed during training.
This avoids its forward and backward computation while directly supervising the visual features it receives at evaluation.

\paragraph{Training configuration.}
The following configuration applies to the four vision--language reward models.
We train for five epochs on unlabeled HPDv2 and AVA images using the resolution buckets
$512\times512$, $768\times768$, $1024\times1024$, $768\times1024$, and $1024\times768$.
These buckets are also used during stitching-layer selection.
Random resize/crop and color jitter are shared between the VAE and reward-model inputs so that they contain the same image content.
For HPSv3, the reward-model input is resized to $448\times448$, producing 256 visual output tokens.
For UnifiedReward, the reward-model input uses the same resolution as the VAE input; its 16-pixel patch grid matches the DiT feature grid at these resolutions.

Optimization uses AdamW with learning rate $5\times10^{-5}$ for all trainable parameters, weight decay $0$, and $100$ linear warmup updates followed by a constant learning rate.
We use global batch size $256$, per-device micro-batch size $8$, and gradient accumulation to obtain the global batch size.
Gradients are clipped at norm $1.0$.
Training uses BF16 autocast with FP32 trainable parameters and a BF16 frozen teacher.
We sample $\sigma\sim\mathrm{Beta}(2,2)$ and construct noisy latents using the forward process in Eq.~\eqref{eq:forward_path}.
Each model is trained for 13,000 optimizer updates.
We maintain an exponential moving average of the trainable parameters with decay $0.999$ and evaluate the fifth-epoch EMA weights.

For UnifiedReward, we remove 30 HPDv2 training images that are near-duplicates of six VL-RewardBench test images.
Matching uses pHash/dHash screening with Hamming distance at most 10, including horizontal flips, followed by grayscale-thumbnail correlation of at least 0.95.
None of these images occurs in the stitching-layer search set; the AVA training set is unchanged.
No benchmark answers or preference labels are used for training.

\paragraph{HPSv3 evaluation.}
We use the original HPSv3 prompt, language model, and scalar reward head for both the original and stitched models.
The original model follows its native aspect-preserving image preprocessing, with a pixel budget of $256\times28\times28$.
For StitchVM, each image is resized bicubically to $1024\times1024$ for VAE encoding.
The diffusion features are mapped to the original HPSv3 image grid through the stitching layer before passing through the retained vision blocks.
The language model therefore receives the same number and arrangement of visual inputs as in the original model.

We report inversion-score accuracy averaged over 400 prompts on HPDv2 and pairwise accuracy averaged over 466 prompts on ImageReward, excluding tied preferences.
The noise levels are $\sigma\in\{0.1,0.25,0.5,0.75,0.9\}$, with fixed noise seeds.
We do not evaluate HPSv3 at $\sigma=0$.

\paragraph{UnifiedReward evaluation.}
We evaluate all 1,247 examples of VL-RewardBench~\citep{li2025vlrewardbench}.
Each example contains an image, a question, and two candidate answers.
The original and stitched models use the same pairwise comparison prompt, chat template, and candidate-answer order.
We use greedy decoding with a maximum of 4,096 generated tokens and disable thinking.
Candidate-answer ordering and latent-noise generation use seed 0.

Both models follow the checkpoint's native image preprocessing, preserving aspect ratio with dimensions divisible by 32 and an image area between 65,536 and 16,777,216 pixels.
For StitchVM, the processed image is encoded by the VAE and perturbed at the specified noise level.
The resulting vision-encoder outputs replace the original visual inputs to the language model, preserving their positions in the input sequence.
Thus, UnifiedReward evaluation uses the processor-native image resolutions rather than the fixed $1024\times1024$ VAE input used for HPSv3.
For images that exceed GPU memory during VAE encoding, the implementation uses CPU FP32 encoding after a GPU retry.
We evaluate $\sigma\in\{0,0.1,0.25,0.5,0.75,0.9\}$.

\subsection{Inference-Time Alignment Experiments}
\label{appendix:inference_time_alignment_experiments}

\paragraph{Shared setup.}
We use SD 3.5 Medium and SD 3.5 Large for DPS, and additionally FLUX.1-dev for FK steering, all at $1024\times 1024$ resolution.
For both DPS and FK steering, we use DrawBench prompts~\citep{saharia2022photorealistic} and report ImageReward, Aesthetic Score, HPSv2, and PickScore~\citep{kirstain2023pick}.
For FK steering, we additionally evaluate compositional alignment on GenEval~\citep{ghosh2023geneval}, and include reference comparisons to the base flow-matching sampler and Best-of-$N$ (BoN) sampling with the same $N=4$.

\paragraph{DPS.}
We use 100 denoising steps with classifier-free guidance scale 4.5 and SDE sampling under LinearSDE of~\citep{kim2026inferencetime}.
The DPS guidance strength is swept independently for standard DPS and DPS with StitchVM, since the two methods produce gradients of different magnitudes: standard DPS backpropagates through the denoiser, VAE, and pixel-space reward, while DPS with StitchVM evaluates the gradient directly in the noisy latent space.
For DPS with StitchVM, we sweep over the range $[0.5, 7.0]$, and for standard DPS, we sweep over the range $[0.01, 0.5]$; in both cases, we report the configuration that maximizes the average of the metrics.
\begin{table*}[t]
\centering
\scriptsize
\setlength{\tabcolsep}{2pt}
\renewcommand{\arraystretch}{1.08}
\caption{\textbf{Results of StitchVM on noisy latents.}
$\oplus$ denotes stitching of a reward model with a pretrained diffusion module
(VAE encoder or DiT). 
}
\label{tab:main_stitchvm_results}
\vspace{-0.1cm}
\begin{subtable}{\textwidth}
\centering
\caption{\textbf{Zero-shot image-text retrieval (Recall@1) on MSCOCO and 
Flickr30K.}}
\vspace{-0.1cm}
\resizebox{\textwidth}{!}{%
\begin{tabular}{l *{5}{c} *{5}{c} *{5}{c} *{5}{c}}
\toprule
& \multicolumn{10}{c}{\textbf{MSCOCO} Recall@1} 
& \multicolumn{10}{c}{\textbf{Flickr30K} Recall@1} \\
\cmidrule(lr){2-11} \cmidrule(lr){12-21}
& \multicolumn{5}{c}{Image$\rightarrow$Text Retrieval} & \multicolumn{5}{c}{Text$\rightarrow$Image Retrieval}
& \multicolumn{5}{c}{Image$\rightarrow$Text Retrieval} & \multicolumn{5}{c}{Text$\rightarrow$Image Retrieval} \\
\cmidrule(lr){2-6}\cmidrule(lr){7-11}\cmidrule(lr){12-16}\cmidrule(lr){17-21}
& \multicolumn{5}{c}{Noise level $\sigma$} & \multicolumn{5}{c}{Noise level $\sigma$}
& \multicolumn{5}{c}{Noise level $\sigma$} & \multicolumn{5}{c}{Noise level $\sigma$} \\
Method
& 0.1 & 0.25 & 0.5 & 0.75 & 0.9
& 0.1 & 0.25 & 0.5 & 0.75 & 0.9
& 0.1 & 0.25 & 0.5 & 0.75 & 0.9
& 0.1 & 0.25 & 0.5 & 0.75 & 0.9 \\
\midrule
NoisyCLIP & 46.07 & 45.16 & 32.71 &  6.14 & 2.11
 & 35.41 & 35.38 & 28.94 &  7.62 & 2.35
 & 76.72 & 73.61 & 49.66 & 10.39 & 2.48
 & 64.48 & 62.27 & 47.58 & 12.04 & 3.18 
 \\
\midrule
\rowcolor{mygrey} \multicolumn{1}{l}{\textit{CLIP ViT-L/14}}
 & \multicolumn{5}{c}{57.90} & \multicolumn{5}{c}{37.09}
 & \multicolumn{5}{c}{87.30} & \multicolumn{5}{c}{67.36} \\
\quad $\oplus$ SD3.5 VAE
& 42.62 & 38.91 & 21.14 &  2.78 & 0.08
 & 31.88 & 29.27 & 17.59 &  4.18 & 0.37
 & 73.40 & 67.25 & 38.30 &  7.10 & 0.40
 & 60.90 & 56.15 & 35.86 &  8.59 & 1.20 \\
\quad $\oplus$ FLUX VAE
 & 39.60 & 36.24 & 24.38 &  4.86 & 0.20
 & 29.69 & 27.70 & 19.10 &  6.23 & 0.84
 & 66.60 & 62.10 & 41.90 &  9.60 & 0.70
 & 56.12 & 52.96 & 38.68 & 13.30 & 1.08 \\
\rowcolor{GoogleBlue!20} \quad $\oplus$ SD3.5-M (StitchVM)
 & 56.82 & 56.22 & 52.50 & 32.52 & 3.68
 & 38.64 & 38.50 & 36.67 & 23.95 & 5.37
 & 86.30 & 85.40 & 82.70 & 58.00 & 8.90
 & 67.66 & 68.28 & 65.62 & 47.20 & 10.46 \\
\rowcolor{GoogleBlue!20} \quad $\oplus$ SD3.5-L (StitchVM)
 & \textbf{57.22} & \textbf{56.90} & 54.36 & 34.10 & 3.62
 & 39.27 & 39.36 & 38.15 & 25.56 & 5.43
 & \textbf{86.80} & \textbf{87.00} & 84.70 & 60.60 & 9.40
 & 69.04 & \textbf{69.70} & 67.88 & 50.88 & 10.88 \\
\rowcolor{GoogleBlue!20} \quad $\oplus$ FLUX  (StitchVM)
 & 56.40 & 56.30 & \textbf{54.46} & \textbf{37.00} & \textbf{6.00}
 & \textbf{39.43} & \textbf{39.89} & \textbf{38.60} & \textbf{28.26} & \textbf{7.61}
 & 86.20 & 86.60 & \textbf{84.80} & \textbf{65.80} & \textbf{15.10}
 & \textbf{69.28} & 68.96 & \textbf{68.58} & \textbf{54.70} & \textbf{16.32} \\
\midrule
\rowcolor{mygrey} \multicolumn{1}{l}{\textit{DFN-CLIP}}
 & \multicolumn{5}{c}{70.62} & \multicolumn{5}{c}{54.16} 
 & \multicolumn{5}{c}{92.20} & \multicolumn{5}{c}{80.86} \\ 
\quad $\oplus$ SD3.5 VAE
 & 57.24 & 54.25 & 37.42 & 9.84 & 1.18
 & 47.53 & 44.77 & 33.46 & 9.97 & 1.14
 & 81.20 & 75.35 & 47.40 & 11.60 & 1.01
 & 74.72 & 69.21 & 50.40 & 11.57 & 1.15 \\
\quad $\oplus$ FLUX VAE
 & 54.22 & 51.58 & 40.66 & 10.92 & 1.31
 & 45.34 & 43.20 & 34.97 & 11.12 & 1.48
 & 74.40 & 70.20 & 51.00 & 12.10 & 1.27
 & 69.94 & 66.02 & 53.22 & 14.28 & 1.19 \\
\rowcolor{GoogleBlue!20} \quad $\oplus$ SD3.5-M (StitchVM)
 & \textbf{71.56} & 71.22 & 68.24 & 46.94 & 6.28
 & 54.22 & 54.21 & 52.34 & 38.04 & 8.95
 & 93.60 & \textbf{93.50} & 91.90 & 74.10 & 13.90
 & \textbf{81.62} & \textbf{81.48} & 79.66 & 62.28 & 14.98 \\
\rowcolor{GoogleBlue!20} \quad $\oplus$ SD3.5-L (StitchVM)
 & 71.44 & \textbf{71.56} & \textbf{68.78} & 48.58 & 7.14
 & 54.29 & 54.00 & 52.54 & 38.84 & 9.19
 & \textbf{94.10} & \textbf{93.50} & 91.80 & 72.50 & 14.10
 & 81.48 & 81.34 & \textbf{80.16} & 62.18 & 14.86 \\
\rowcolor{GoogleBlue!20} \quad $\oplus$ FLUX (StitchVM)
 & 71.30 & 70.84 & 68.24 & \textbf{49.38} & \textbf{8.58}
 & \textbf{54.39} & \textbf{54.24} & \textbf{52.84} & \textbf{39.93} & \textbf{11.78}
 & 93.20 & 93.10 & \textbf{92.00} & \textbf{74.60} & \textbf{18.70}
 & 81.52 & 81.28 & \textbf{80.16} & \textbf{63.84} & \textbf{19.00} \\
\bottomrule
\end{tabular}
}
\vspace{0.1cm}
\end{subtable}
\begin{subtable}[t]{0.57\textwidth}
\centering
\caption{\textbf{Preference accuracy on HPDv2 and ImageReward.}}
\vspace{-0.1cm}
\label{tab:main_stitchvm_hpsv2}
\resizebox{\textwidth}{!}{%
\begin{tabular}{l *{5}{c} *{5}{c}}
\toprule
& \multicolumn{5}{c}{\textbf{HPDv2 Benchmark}}
& \multicolumn{5}{c}{\textbf{ImageReward Benchmark}} \\
\cmidrule(lr){2-6}\cmidrule(lr){7-11}
& \multicolumn{5}{c}{Noise level $\sigma$}
& \multicolumn{5}{c}{Noise level $\sigma$} \\
Method
& 0.1 & 0.25 & 0.5 & 0.75 & 0.9
& 0.1 & 0.25 & 0.5 & 0.75 & 0.9 \\
\midrule
DiNa-LRM~\cite{liu2026beyond}
 & 81.36 & 81.41 & 80.69 & 78.16 & 65.18
 & 60.61 & 60.83 & 60.61 & 59.21 & 54.13 \\
\midrule
\rowcolor{mygrey} \multicolumn{1}{l}{\textit{HPSv2}}
 & \multicolumn{5}{c}{83.28} & \multicolumn{5}{c}{67.14} \\
\quad $\oplus$ SD3 VAE
 & 79.67 & 79.22 & 76.44 & 71.49 & 62.77
 & 60.72 & 59.56 & 57.89 & 54.59 & 53.00 \\
\quad $\oplus$ FLUX VAE
 & 79.60 & 79.12 & 77.96 & 73.02 & 65.23
 & 62.27 & 62.48 & 60.76 & 58.14 & 52.56 \\
\rowcolor{GoogleBlue!20} \quad $\oplus$ SD3.5-M (StitchVM)
 & 81.78 & 81.87 & 80.92 & 78.90 & 74.68
 & \textbf{66.82} & 67.00 & 65.25 & 63.33 & 58.19 \\
\rowcolor{GoogleBlue!20} \quad $\oplus$ SD3.5-L (StitchVM)
 & 82.03 & 82.09 & 81.40 & 78.98 & 75.39
 & 66.68 & \textbf{67.28} & 66.38 & 63.04 & 58.71 \\
\rowcolor{GoogleBlue!20} \quad $\oplus$ FLUX  (StitchVM)
 & \textbf{82.53} & \textbf{82.61} & \textbf{81.98} & \textbf{79.25} & \textbf{75.90}
 & 66.23 & 66.49 & \textbf{66.67} & \textbf{64.77} & \textbf{59.04} \\
\bottomrule
\end{tabular}}
\end{subtable}
\hfill
\begin{subtable}[t]{0.42\textwidth}
\centering
\caption{\textbf{Aesthetic-score correlation on AVA.}}
\vspace{-0.1cm}
\label{tab:main_stitchvm_aesthetic}
\resizebox{\textwidth}{!}{%
\begin{tabular}{l *{5}{c}}
\toprule
& \multicolumn{5}{c}{\textbf{SRCC}} \\
\cmidrule(lr){2-6}
& \multicolumn{5}{c}{Noise level $\sigma$} \\
Method
& 0.1 & 0.25 & 0.5 & 0.75 & 0.9 \\
\midrule
\rowcolor{mygrey} \multicolumn{1}{l}{\textit{Aesthetic Predictor}}
 & \multicolumn{5}{c}{0.618} \\
\quad $\oplus$ SD3 VAE
 & 0.438 & 0.423 & 0.334 & 0.146 & 0.068 \\
\quad $\oplus$ FLUX VAE
 & 0.455 & 0.451 & 0.415 & 0.322 & 0.177 \\
\rowcolor{GoogleBlue!20} \quad $\oplus$ SD3.5-M (StitchVM)
 & 0.609 & 0.610 & 0.597 & 0.538 & 0.369 \\
\rowcolor{GoogleBlue!20} \quad $\oplus$ SD3.5-L (StitchVM)
 & \textbf{0.614} & 0.615 & \textbf{0.601} & 0.545 & 0.396 \\
\rowcolor{GoogleBlue!20} \quad $\oplus$ FLUX  (StitchVM)
 & 0.613 & \textbf{0.616} & 0.600 & \textbf{0.560} & \textbf{0.433} \\
\bottomrule
\end{tabular}}
\end{subtable}
\end{table*}

\paragraph{FK steering.}
We use the default sampling configuration of each base generator.
For SD~3.5 Medium and SD~3.5 Large, we use 40 denoising steps with classifier-free guidance~\citep{ho2022classifier} scale 4.5.
For FLUX.1-dev, we use 28 denoising steps with guidance scale 3.5.
For SDE sampling, we use the LinearSDE following~\citep{kim2026inferencetime}.

For standard FK steering, we apply resampling every 4 denoising steps within a fixed active window.
For SD~3.5 Medium and Large, this window spans denoising steps 8 through 32 out of 40 total steps.
For FLUX, it spans denoising steps 6 through 22 out of 28 total steps.
For FK steering with StitchVM, we apply proposal scaling on a separate schedule from FK resampling: it is applied at every step in the early high-noise portion of sampling, covering approximately the first $40\%$ of the denoising trajectory.
In all FK steering experiments with StitchVM, we use $N=4$ particles and $M=2$ proposals per particle unless otherwise specified.
Other settings follow the default configuration of FK steering.

\subsection{Training-Time Alignment Experiments}
\label{appendix:training_time_alignment_experiments}

\paragraph{Shared setup.}
For both direct reward finetuning and DiffusionNFT, we follow the protocol of~\citep{xue2025dancegrpo}.
The base generator is SD3.5~Medium, and training prompts are drawn from HPDv2~\citep{wu2023human}.
The training reward is defined as the equal-weight sum of DFN-CLIP and HPSv2 rewards:
\begin{equation}
    \mathcal{R}(\zb_0)
    =
    \mathcal{R}_{\mathrm{DFN\text{-}CLIP}}(\zb_0)
    +
    \mathcal{R}_{\mathrm{HPSv2}}(\zb_0),
\end{equation}
where $\mathcal{R}_{\mathrm{DFN\text{-}CLIP}}$ denotes the DFN-CLIP image-text alignment score and $\mathcal{R}_{\mathrm{HPSv2}}$ denotes the HPSv2 human preference score.
Both rewards are evaluated on the decoded clean image, and we use the shorthand $\mathcal{R}(\zb_0)=\mathcal{R}(\Dc(\zb_0))$.
For variants with StitchVM, we replace the terminal clean reward with the corresponding StitchVM prediction.
All training runs use $4$ nodes $\times$ $4$ NVIDIA GH200 GPUs, for $16$ GPUs in total, at $512{\times}512$ resolution.

\paragraph{Direct reward finetuning setup.}
For all direct reward finetuning methods, we finetune the model with LoRA~\citep{hu2022lora}.
We use LoRA rank $r=32$ and $\alpha=64$.
LoRA is applied to all attention projections of the joint MM-DiT blocks: \texttt{to\_q}, \texttt{to\_k}, \texttt{to\_v}, \texttt{to\_out.0}, \texttt{add\_q\_proj}, \texttt{add\_k\_proj}, \texttt{add\_v\_proj}, and \texttt{to\_add\_out}.
The trainable parameters are optimized with AdamW~\citep{loshchilov2018decoupled} using weight decay $0$, learning rate $5\!\times\!10^{-5}$, gradient clipping at norm $1.0$, and EMA on the trainable parameters.
The regularization weight in Eqs.~\eqref{eq:appendix_drft_objective}--\eqref{eq:appendix_vdrft_objective} is $\lambda=0.01$.
Training sampling uses a $10$-step rectified-flow schedule with shift $3.0$ at $\mathrm{cfg}=1.0$.
Per training step, we draw $8$ prompts per device with $16$ gradient-accumulation steps and run $32$ sample-and-update batches per epoch.
Across all $16$ GPUs, this gives an effective batch size of $8\!\times\!16\!\times\!16=2048$ samples per optimizer update.
For DRaFT with StitchVM, the stopping step in Eq.~\eqref{eq:appendix_vdrft_objective} is sampled uniformly from the index window $[3,7]$ of the $10$-step schedule.

\paragraph{DiffusionNFT setup.}
We follow the multi-reward setup from the official DiffusionNFT implementation\footnote{\url{https://github.com/NVlabs/DiffusionNFT}}, except that we use the joint DFN-CLIP and HPSv2 reward defined above.
For DiffusionNFT with StitchVM, each rollout is stopped early at a step sampled uniformly from $\{12,\dots,17\}$ of the $25$-step schedule.

\paragraph{Flow-GRPO-Fast setup.}
We follow the PickScore setup from the official Flow-GRPO-Fast implementation\footnote{\url{https://github.com/yifan123/flow_grpo}}, except that we replace the reward with the joint DFN-CLIP and HPSv2 reward defined above and run the method on the same $16$-GPU setup as the other training-time alignment experiments.

\paragraph{Evaluation protocol.}
For sample generation, we use $40$ denoising steps at $\mathrm{cfg}=1.0$ for all methods.
All methods are evaluated on fully denoised samples using the original clean-image reward models, including the variants with StitchVM.
We report total GPU-hours per run in the ``GPU-h'' column of Table~\ref{tab:rl_finetune} as a wall-clock training-cost measure.
All timings are measured on the same $16$-GH200 layout.

\section{Additional Experimental Results}
\label{appendix:additional_results}

\subsection{Full Numerical Results of StitchVM Performance}
\label{appendix:stitchvm_full_results}
\begin{figure}[t]
\centering
\includegraphics[width=\linewidth]{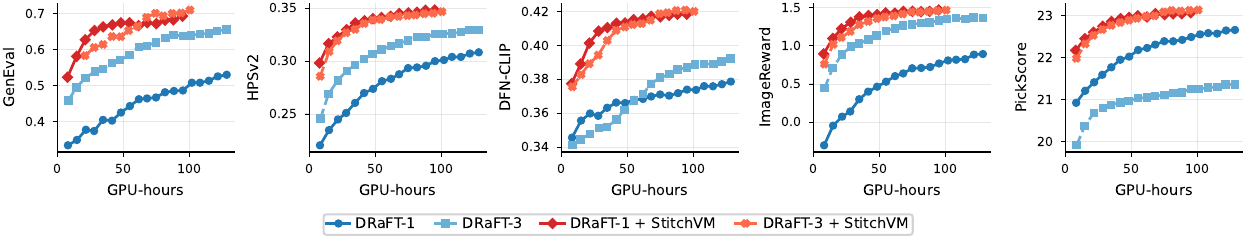}
\caption{\textbf{StitchVM improves the quality--efficiency trade-off of DRaFT.}
GenEval, HPSv2, DFN-CLIP, ImageReward, and PickScore plotted against GPU-hours for DRaFT and DRaFT with StitchVM during finetuning on the joint DFN-CLIP and HPSv2 reward.}
\label{fig:rl_convergence_appendix_draft}
\end{figure}
\begin{figure}[t]
\centering
\includegraphics[width=\linewidth]{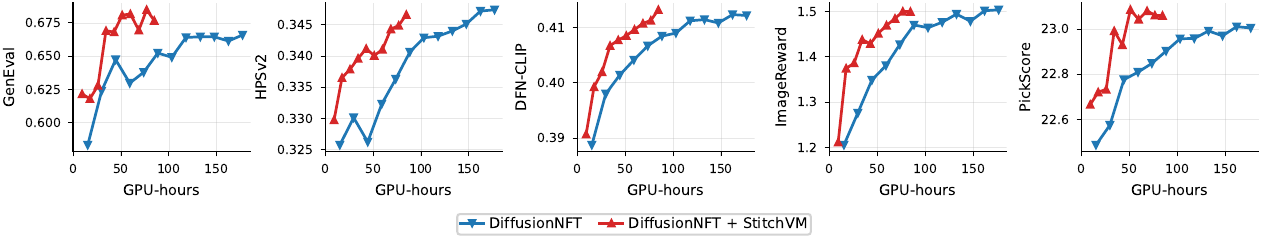}
\caption{\textbf{StitchVM accelerates DiffusionNFT training.}
GenEval, HPSv2, DFN-CLIP, ImageReward, and PickScore plotted against GPU-hours for DiffusionNFT and DiffusionNFT with StitchVM during finetuning on the joint DFN-CLIP and HPSv2 reward.}
\label{fig:rl_convergence_appendix_nft}
\end{figure}

For completeness, Table~\ref{tab:main_stitchvm_results} reports the full numerical results corresponding to the line plots in Figure~\ref{fig:stitchvm_results}.
The table is organized into three evaluation settings:
(a) zero-shot image-text retrieval on MSCOCO~\citep{lin2014microsoft} and Flickr30K~\citep{young2014image}, evaluated with Recall@1 in both Image$\to$Text and Text$\to$Image directions, for CLIP ViT-L/14~\citep{radford2021learning} and DFN-CLIP~\citep{fang2024data};
(b) preference accuracy on HPDv2~\citep{wu2023human} and ImageReward~\citep{xu2023imagereward}, for HPSv2~\citep{wu2023human}; and
(c) aesthetic score correlation on the AVA test split~\citep{murray2012ava}, evaluated with SRCC, for the Aesthetic Predictor~\citep{schuhmann2022improvedaestheticpredictor}.
Each setting reports results at noise levels $\sigma\in\{0.1,\, 0.25,\, 0.5,\, 0.75,\, 0.9\}$.

\subsection{Training Curves in Training-Time Alignment}
\label{appendix:additional_results_rl_convergence}

Figures~\ref{fig:rl_convergence_appendix_draft} and~\ref{fig:rl_convergence_appendix_nft} report training curves for DRaFT and DiffusionNFT during finetuning on the joint DFN-CLIP and HPSv2 reward.
We plot GenEval, HPSv2, DFN-CLIP, ImageReward, and PickScore against GPU-hours to evaluate both training-reward metrics and held-out metrics.
For DRaFT (Fig.~\ref{fig:rl_convergence_appendix_draft}), variants with StitchVM reach higher final scores with substantially less compute across metrics, showing an improved quality--efficiency trade-off.
For DiffusionNFT (Fig.~\ref{fig:rl_convergence_appendix_nft}), variants with StitchVM reach a similar plateau to the original method in roughly half the GPU-hours, indicating faster training while preserving final quality.
Overall, these curves show that the gains from StitchVM are consistent across compositional alignment, training rewards, and held-out preference metrics.

\begin{figure}[t]
\centering
\includegraphics[width=\linewidth]{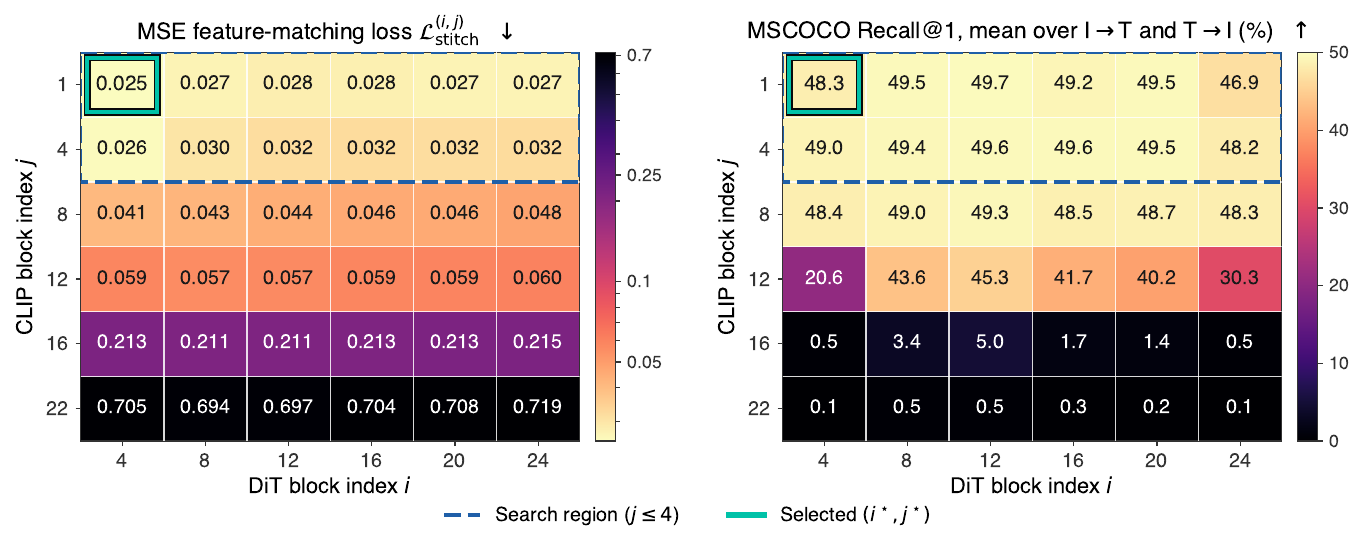}
\caption{\textbf{Stitching interface analysis on CLIP ViT-L/14 with SD~3.5~Medium.}
We sweep the DiT block index $i$ and CLIP block index $j$, fit the closed-form stitching map, and then run Stage-2 training.
\textit{Left}: closed-form MSE feature-matching loss.
\textit{Right}: MSCOCO Recall@1 averaged over image-to-text and text-to-image retrieval at $\sigma=0.1$ after Stage-2 training.
The closed-form loss filters out catastrophic interfaces, especially for later CLIP blocks where Recall@1 drops sharply, but is weakly aligned with the final score within the early low-loss region.
Based on this observation, we restrict the search to $j \le 4$ (dashed box) and select the lowest-loss cell $(i^\star,j^\star)=(4,1)$.}
\label{fig:stitching_layer_ablation}
\end{figure}

\subsection{Analysis of Stitching Interface Search}
\label{appendix:stitching_layer_analysis}

The closed-form feature-matching loss in Eq.~\eqref{eq:stitch_fit} is intended as a cheap search protocol for the stitch interface $(i,j)$.
To test this, we sweep $(i,j)$ over the full grid, fit $W^\star_{i,j}$ in closed form, and then train the stitched value model end-to-end.
Figure~\ref{fig:stitching_layer_ablation} reports the closed-form feature-matching loss alongside the MSCOCO Recall@1 of the trained model.

The closed-form loss sharply filters out poor interfaces but does not precisely identify the best one.
Once $j$ moves beyond the early CLIP blocks, the loss rises by roughly an order of magnitude, and Recall@1 drops from around $49$ to below $5$; Stage-2 finetuning cannot recover from these poor interfaces.
In contrast, within the low-loss region ($j \le 4$), Recall@1 remains uniformly high and only weakly correlates with the loss: the lowest-loss cell $(i,j)=(4,1)$ reaches Recall@1 of $48.3$, compared to a within-region maximum of $49.7$.

We exploit this asymmetry in our search.
Since high-loss configurations cannot be recovered through finetuning, we restrict the reward model cut to the early CLIP blocks ($j \le 4$; Appendix~\ref{appendix:svm_details}) and select the lowest-loss cell within this range.
This avoids catastrophic stitch points at a cost of about $1.5$ Recall@1 points relative to the within-region maximum, in exchange for a much cheaper search.

\subsection{Cross-Backbone Generalization of StitchVM}
\label{appendix:smallgenerator}
Many diffusion and flow-based generators~\citep{labs2025flux, stabilityai2024sd35} are released at multiple model scales while sharing the same VAE latent space.
Here, we ask whether a StitchVM trained on a smaller backbone can guide a larger generator, and how much performance is lost relative to a StitchVM trained on the same backbone in FK steering.

We keep the FK steering setup of Section~\ref{sec:results_inference_time} on SD~3.5~Large, but replace the SD~3.5~Large StitchVM with a StitchVM trained on SD~3.5~Medium.
SD~3.5~Medium and SD~3.5~Large share the same SD3 16-channel VAE, so their noisy latents are dimensionally compatible at every noise level.
Table~\ref{tab:xback_results} reports the same five metrics as Table~\ref{tab:fk_steering}.

Across the fifteen reward--metric cells, the SD~3.5~Medium StitchVM closely matches the SD~3.5~Large StitchVM: HPSv2 differs by at most $0.002$, the other metrics remain similar, and under HPSv2 reward, the SD~3.5~Medium StitchVM achieves higher GenEval than the SD~3.5~Large StitchVM ($0.72$ vs.\ $0.70$).
Because StitchVM uses only the early blocks of its diffusion backbone, building it on SD~3.5~Medium rather than SD~3.5~Large reduces the per-step value model cost.
A StitchVM trained on the smaller backbone can therefore guide the larger generator at lower inference cost with little loss in alignment quality.
This property can further reduce the cost of $M$-scaling in FK steering, since a smaller StitchVM makes each additional proposal cheaper while still effectively guiding the larger generator.

\subsection{Stopping-Step Distribution for RL Finetuning with StitchVM}
\label{appendix:stop_step_distribution}

The training-time methods in Section~\ref{sec:method_training} stop the rollout at an intermediate noisy latent $\zb_\tau$ and use the StitchVM value function $V_\omega^{(i^\star,j^\star)}(\zb_\tau)$ in place of the terminal reward.
Here, we ablate the distribution from which the stopping step is sampled.
For the DiffusionNFT stopping-step ablation, we use a
$25$-step denoising schedule. Smaller step indices correspond
to earlier, higher-noise latents, while larger step indices
correspond to later, cleaner latents.

On this schedule, we sample the stopping step uniformly from one of four windows: a high-noise window $\mathrm{Unif}\{2,\dots,12\}$, a tight intermediate window $\mathrm{Unif}\{12,\dots,17\}$, a wide intermediate-to-low-noise window $\mathrm{Unif}\{12,\dots,25\}$, and a low-noise window $\mathrm{Unif}\{20,\dots,25\}$.
The tight intermediate window is our default.

Figure~\ref{fig:kstop_window_ablation} reports GenEval, HPSv2, DFN-CLIP, ImageReward, and PickScore as a function of GPU-hours.
The high-noise window $\mathrm{Unif}\{2,\dots,12\}$ consistently underperforms, suggesting that stopping too early in the denoising trajectory yields value function targets that are less useful for finetuning.
In contrast, including intermediate-noise latents substantially improves performance.
Among the tested choices, $\mathrm{Unif}\{12,\dots,17\}$ provides the best quality--efficiency trade-off: it reaches strong final scores across metrics and converges faster than the wider window $\mathrm{Unif}\{12,\dots,25\}$.
The low-noise window $\mathrm{Unif}\{20,\dots,25\}$ remains competitive on some metrics, such as PickScore, but is less stable overall.

We therefore use $\mathrm{Unif}\{12,\dots,17\}$ as the default stopping-step distribution for all DiffusionNFT and DRaFT runs with StitchVM in Table~\ref{tab:rl_finetune}.

\subsection{Qualitative Results on RL Finetuning with StitchVM}
Figures~\ref{fig:qual_rl_draft_1}, \ref{fig:qual_rl_draft_3}, and~\ref{fig:qual_rl_diffusionNFT} show qualitative comparisons of DRaFT-1, DRaFT-3, and DiffusionNFT, with and without StitchVM, across training GPU-hours.
Across all three methods, the StitchVM-augmented variants reach the target prompt earlier and produce visually higher-quality samples throughout training (e.g., sharper details and more saturated colors, characteristic of HPSv2-tuned outputs).

\begin{table*}[t]
\centering
\scriptsize
\setlength{\tabcolsep}{4pt}
\renewcommand{\arraystretch}{1.15}
\caption{\textbf{A StitchVM with a smaller generator backbone can guide a larger generator with only minor degradation.}
We apply FK steering to SD3.5-Large using either an SD3.5-Large StitchVM (same-backbone) or an SD3.5-Medium StitchVM (cross-backbone).
ImgRwd: ImageReward, Aes: Aesthetic, Pick: PickScore.}
\label{tab:xback_results}
\vspace{-0.1cm}
\resizebox{0.78\textwidth}{!}{%
\begin{tabular}{l c c c c c}
\toprule
StitchVM backbone
 & ImgRwd $\uparrow$ & Aes $\uparrow$ & HPSv2 $\uparrow$
 & Pick $\uparrow$ & GenEval $\uparrow$ \\
\midrule
\rowcolor{mygrey} \multicolumn{6}{l}{\textit{Target reward: HPSv2 score}} \\
\quad SD3.5-Large (same-backbone)
 & \textbf{1.20} & \textbf{5.52} & \textbf{0.310} & \textbf{23.11} & 0.70 \\
\rowcolor{GoogleBlue!20} \quad SD3.5-Medium (cross-backbone)
 & 1.17 & 5.50 & 0.309 & 23.06 & \textbf{0.72} \\
\midrule
\rowcolor{mygrey} \multicolumn{6}{l}{\textit{Target reward: Aesthetic score}} \\
\quad SD3.5-Large (same-backbone)
 & 1.03 & \textbf{5.68} & \textbf{0.298} & 22.87 & \textbf{0.68} \\
\rowcolor{GoogleBlue!20} \quad SD3.5-Medium (cross-backbone)
 & \textbf{1.05} & 5.65 & 0.297 & \textbf{22.91} & 0.67 \\
\midrule
\rowcolor{mygrey} \multicolumn{6}{l}{\textit{Target reward: CLIP score}} \\
\quad SD3.5-Large (same-backbone)
 & \textbf{1.20} & 5.40 & \textbf{0.298} & \textbf{22.96} & \textbf{0.71} \\
\rowcolor{GoogleBlue!20} \quad SD3.5-Medium (cross-backbone)
 & 1.10 & \textbf{5.42} & 0.296 & 22.88 & 0.69 \\
\bottomrule
\end{tabular}}
\vspace{-0.2cm}
\end{table*}

\begin{figure}[t]
\centering
\includegraphics[width=\linewidth]{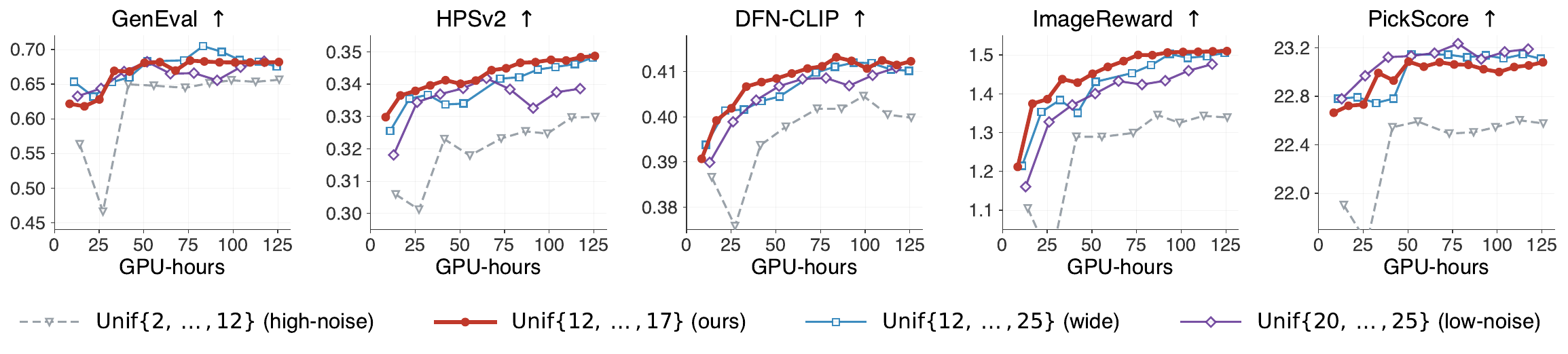}
\caption{\textbf{Effect of the stopping-step distribution in DiffusionNFT with StitchVM.}
We train DiffusionNFT with StitchVM on SD~3.5~Medium with the joint DFN-CLIP and HPSv2 reward, varying the window from which the stopping step is sampled uniformly over a 25-step denoising schedule.
Smaller step indices correspond to earlier, higher-noise latents, while larger indices correspond to later, lower-noise latents.
}
\label{fig:kstop_window_ablation}
\end{figure}

\begin{figure}[t]
\centering
\includegraphics[width=\linewidth]{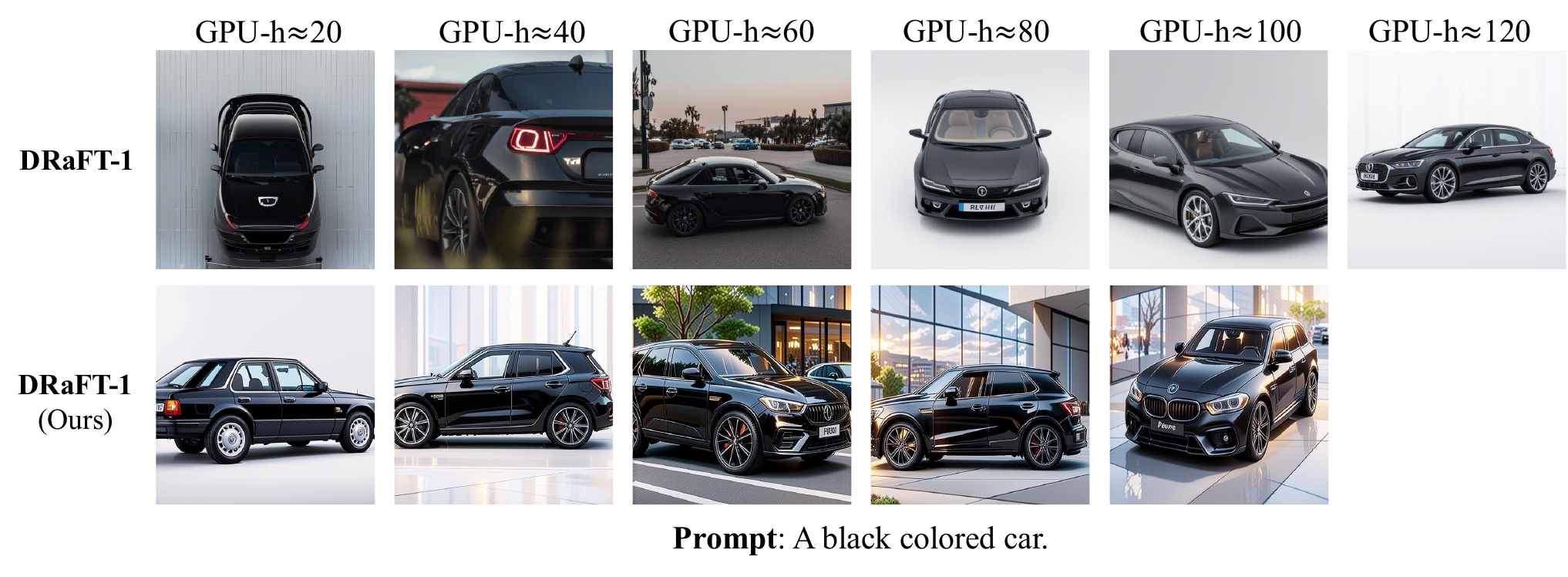}
\vspace{-0.6cm}
\caption{\textbf{Qualitative comparison between DRaFT-1 and DRaFT-1 with StitchVM across training GPU-hours.}}
\label{fig:qual_rl_draft_1}
\end{figure}
\begin{figure}[t]
\centering
\includegraphics[width=\linewidth]{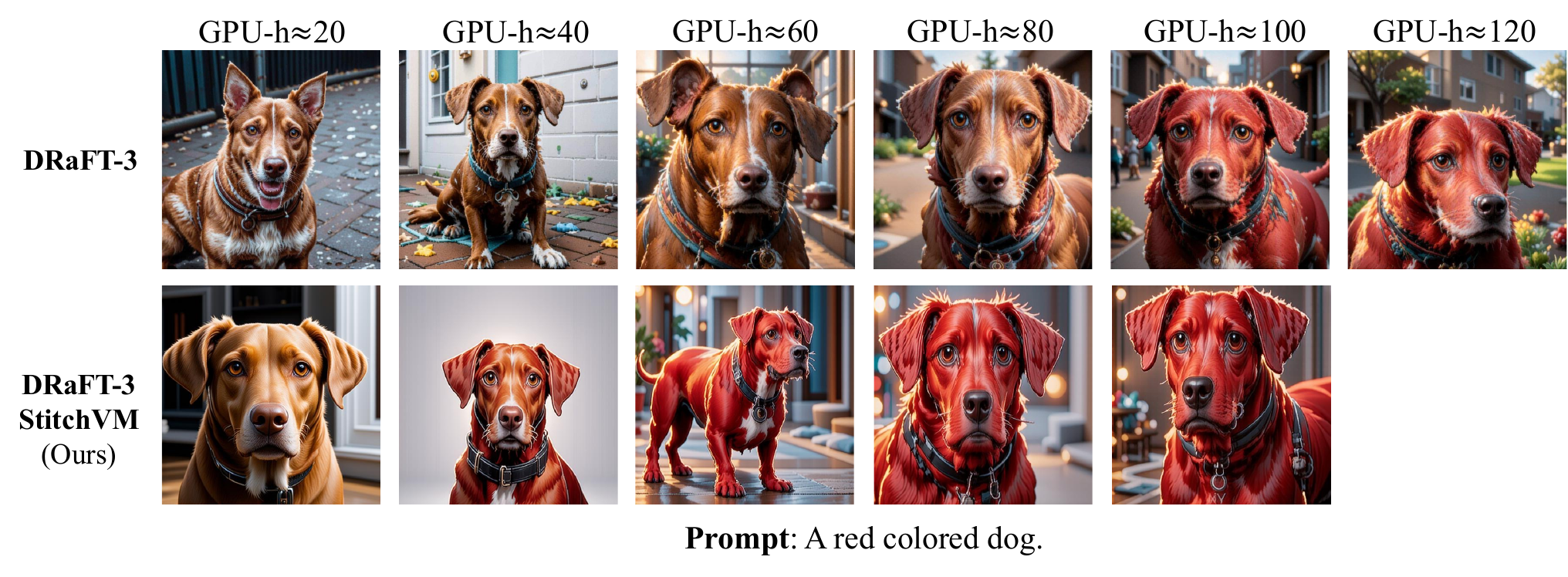}
\vspace{-0.6cm}
\caption{\textbf{Qualitative comparison between DRaFT-3 and DRaFT-3 with StitchVM across training GPU-hours.}}
\label{fig:qual_rl_draft_3}
\end{figure}
\begin{figure}[t]
\centering
\includegraphics[width=\linewidth]{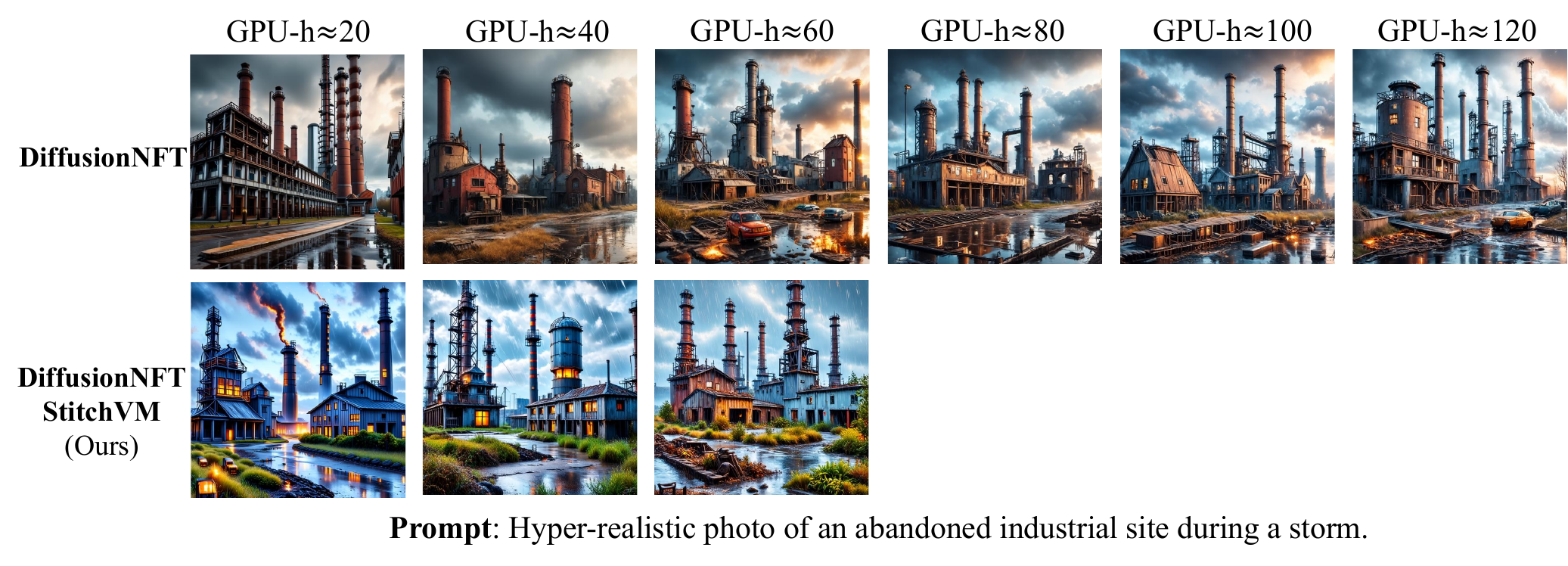}
\vspace{-0.6cm}
\caption{\textbf{Qualitative comparison between DiffusionNFT and DiffusionNFT with StitchVM across training GPU-hours.}}
\label{fig:qual_rl_diffusionNFT}
\end{figure}
\section{Limitation}
\label{appendix:limitation}

StitchVM enables transferring feedforward-model-based rewards to noisy latents, but it does not directly apply to rewards that are not implemented as feedforward models. 
We believe this limitation could be addressed by training surrogate reward models for such rewards, but we leave this direction to the future scope of this work.

\paragraph{Future directions.}

In this work, we focus on a simple method rather than more complex alternatives that may further improve performance.
In our view, timestep-aware training methods~\citep{go2023towards, park2024switch, lee2024multi, park2024denoising, ham2025diffusion}, which have demonstrated effectiveness in diffusion models, represent a promising future direction for improving StitchVM.
Recent follow-up work~\cite{kim2026valopvaluegradientlookahead} shows a promising direction to improve the alignment quality with StitchVM. 
We hope that improved StitchVM with community efforts replaces Tweedie estimation practices and offers alternative advantages over them.


\section{Broader Impacts}
\label{appendix:broader_impacts}

\paragraph{Potential positive impacts.}
We present StitchVM, a practical framework for training noisy latent value models from existing pretrained reward models. 
By making value-model training substantially cheaper, StitchVM can improve and accelerate reward-based alignment methods for diffusion and flow models. 
This may help make generative models more controllable, more aligned with human preferences, and easier to adapt to downstream tasks where clean-image reward models are already available. 
More broadly, our work suggests a practical direction for value-model-based alignment, which could further accelerate research on safer and more reliable generative modeling.

\paragraph{Potential negative impacts.}
The same improvements in controllability and reward optimization may also be misused. 
For example, stronger alignment and steering methods could be used to generate more persuasive synthetic images, including misleading or deceptive visual content. 
StitchVM may make it easier to optimize diffusion-based generation towards all kinds of rewards, including poorly specified or harmful ones. 
These risks are not unique to our method, but our work could lower the cost of applying reward-based steering and post-training. 
We therefore recommend the use of StitchVM together with appropriate safeguards for misuse, bias, and harmful content generation.



\end{document}